\documentclass[twoside,11pt]{article}

\PassOptionsToPackage{hyphens}{url}

\usepackage{algorithm}
\usepackage{algpseudocode}

\usepackage[preprint,abbrvbib]{jmlr2e}

\makeatletter
\@for\jmlr@env:={theorem,lemma,proposition,remark,corollary,definition,%
                 example,conjecture,axiom}\do{%
  \expandafter\let\csname\jmlr@env\endcsname\relax
  \expandafter\let\csname end\jmlr@env\endcsname\relax
  \expandafter\let\csname c@\jmlr@env\endcsname\relax
  \expandafter\let\csname the\jmlr@env\endcsname\relax
  \expandafter\let\csname p@\jmlr@env\endcsname\relax}
\let\proof\relax 
\makeatother

\makeatletter
\long\def\@makecaption#1#2{%
  \vskip 10pt
  \sbox\@tempboxa{#1: #2}%
  \ifdim\wd\@tempboxa>\hsize
    \begingroup
      \sbox\@tempboxa{#1:}%
      \list{\usebox\@tempboxa}{%
        \setlength{\labelwidth}{\wd\@tempboxa}%
        \setlength{\leftmargin}{\labelwidth}%
        \addtolength{\leftmargin}{\labelsep}}%
      \item\relax #2\endlist\par
    \endgroup
  \else
    \hbox to\hsize{\hfil\box\@tempboxa\hfil}%
  \fi}
\makeatother

\usepackage{times}
\usepackage[T1]{fontenc}
\usepackage{amsmath,amsthm,mathtools}
\usepackage{array}
\usepackage{booktabs}
\usepackage{multirow}
\usepackage{xcolor}
\usepackage{enumitem}
\usepackage{microtype}
\hypersetup{
  pdftitle={Counterfactual Evaluation of Temporal Observation Protocols},
  pdfauthor={Xizhe Zhang},
  pdfsubject={Identifiability, calibration and design of the predictive value
              of undeployed temporal observation protocols},
  pdfkeywords={identifiability, observational equivalence, observation design,
               temporal aggregates, calibration}}
\usepackage[section]{placeins}
\usepackage{cleveref}

\usepackage{aliascnt}
\theoremstyle{plain}
\newtheorem{theorem}{Theorem}
\newaliascnt{proposition}{theorem}
\newtheorem{proposition}[proposition]{Proposition}
\aliascntresetthe{proposition}
\newaliascnt{lemma}{theorem}
\newtheorem{lemma}[lemma]{Lemma}
\aliascntresetthe{lemma}
\newaliascnt{corollary}{theorem}
\newtheorem{corollary}[corollary]{Corollary}
\aliascntresetthe{corollary}
\theoremstyle{definition}
\newaliascnt{definition}{theorem}
\newtheorem{definition}[definition]{Definition}
\aliascntresetthe{definition}
\newaliascnt{assumption}{theorem}

\aliascntresetthe{assumption}
\newaliascnt{example}{theorem}
\newtheorem{example}[example]{Example}
\aliascntresetthe{example}
\theoremstyle{remark}
\newaliascnt{remark}{theorem}

\aliascntresetthe{remark}

\Crefname{theorem}{Theorem}{Theorems}
\Crefname{proposition}{Proposition}{Propositions}
\Crefname{lemma}{Lemma}{Lemmas}
\Crefname{corollary}{Corollary}{Corollaries}
\Crefname{assumption}{Assumption}{Assumptions}
\Crefname{definition}{Definition}{Definitions}
\Crefname{remark}{Remark}{Remarks}
\Crefname{example}{Example}{Examples}
\Crefname{algorithm}{Algorithm}{Algorithms}

\crefname{theorem}{Theorem}{Theorems}
\crefname{proposition}{Proposition}{Propositions}
\crefname{lemma}{Lemma}{Lemmas}
\crefname{corollary}{Corollary}{Corollaries}
\crefname{assumption}{Assumption}{Assumptions}
\crefname{definition}{Definition}{Definitions}
\crefname{remark}{Remark}{Remarks}
\crefname{example}{Example}{Examples}
\crefname{algorithm}{Algorithm}{Algorithms}
\crefname{section}{Section}{Sections}
\crefname{subsection}{Section}{Sections}
\crefname{figure}{Figure}{Figures}
\crefname{table}{Table}{Tables}
\Crefname{section}{Section}{Sections}
\Crefname{subsection}{Section}{Sections}
\Crefname{figure}{Figure}{Figures}
\Crefname{table}{Table}{Tables}

\newcommand{\E}{\mathbb{E}}
\newcommand{\Prob}{\mathbb{P}}
\newcommand{\Var}{\operatorname{Var}}
\newcommand{\Cov}{\operatorname{Cov}}

\newcommand{\R}{\mathbb{R}}

\newcommand{\Ind}{\mathbf{1}}
\newcommand{\cN}{\mathcal{N}}
\newcommand{\cI}{\mathcal{I}}

\newcommand{\cD}{\mathcal{D}}

\newcommand{\cV}{\mathcal{V}}
\newcommand{\cR}{\mathcal{R}}

\newcommand{\cT}{\mathcal{T}}

\newcommand{\dd}{\,\mathrm{d}}
\newcommand{\opnorm}[1]{\left\lVert #1 \right\rVert_{\mathrm{op}}}

\DeclareMathOperator{\tr}{tr}
\DeclareMathOperator{\diag}{diag}
\DeclareMathOperator{\Diag}{Diag}
\DeclareMathOperator{\proj}{proj}
\DeclareMathOperator{\rank}{rank}
\DeclareMathOperator{\rowspace}{row}

\newcommand{\numIdEps}{0.1321}

\newcommand{\numIdDiscrepancy}{$10^{-16}$}

\newcommand{\numIdCeilPlus}{0.6817}
\newcommand{\numIdCeilMinus}{0.8274}

\newcommand{\numIdCovYTheta}{0.388250}
\newcommand{\numIdVarTheta}{0.428012}

\newcommand{\numEstReplications}{200}

\newcommand{\numEstFamily}{495}

\newcommand{\numSlopeMean}{-0.4134}

\newcommand{\numTailSlopeMean}{-0.4618}

\newcommand{\numSlopeOccZero}{-0.4142}

\newcommand{\numTailSlopeOccZero}{-0.4637}

\newcommand{\numAugStatBefore}{1}
\newcommand{\numAugStatAfter}{0}
\newcommand{\numAugFreeBefore}{4}
\newcommand{\numAugFreeOne}{2}
\newcommand{\numAugFreeTwo}{0}

\newcommand{\numSleepGrid}{128}
\newcommand{\numAfGrid}{96}

\newcommand{\numSwapRounds}{3}

\newcommand{\numTabTrainSubjects}{80}

\newcommand{\numSelectionSubsampleReps}{1000}
\newcommand{\numSelectionSubsamplePct}{80\%}
\newcommand{\numSweepMSmall}{20}
\newcommand{\numSweepDeltaKqSmall}{-0.006}
\newcommand{\numSweepDeltaUniSmall}{-0.004}

\newcommand{\numSweepDeltaKqFull}{+0.044}
\newcommand{\numSweepDeltaUniFull}{-0.044}

\newcommand{\numSelectionDeltaKqMed}{-0.014}
\newcommand{\numSelectionDeltaKqPlo}{-0.102}
\newcommand{\numSelectionDeltaKqPhi}{+0.078}

\newcommand{\numSelectionDeltaUniMed}{-0.012}
\newcommand{\numSelectionDeltaUniPlo}{-0.126}
\newcommand{\numSelectionDeltaUniPhi}{+0.097}

\newcommand{\numRegretRatioMax}{0.37}
\newcommand{\numRegretRatioMedian}{0.024}

\newcommand{\numRegretSlope}{-0.94}
\newcommand{\numEpsSlope}{-0.49}

\newcommand{\numDesignComparableInstances}{25}
\newcommand{\numDesignMethodMinMI}{0.601}
\newcommand{\numDesignMethodMeanMI}{0.864}
\newcommand{\numDesignMethodMedianMI}{0.852}
\newcommand{\numDesignMethodMinIMSE}{0.802}
\newcommand{\numDesignMethodMeanIMSE}{0.947}
\newcommand{\numDesignMethodMedianIMSE}{0.987}
\newcommand{\numDesignMethodMinLinear}{0.767}
\newcommand{\numDesignMethodMeanLinear}{0.982}
\newcommand{\numDesignMethodMedianLinear}{1.000}
\newcommand{\numDesignMethodMinKQ}{0.231}
\newcommand{\numDesignMethodMeanKQ}{0.899}
\newcommand{\numDesignMethodMedianKQ}{0.988}
\newcommand{\numDesignMethodMinGreedy}{0.912}
\newcommand{\numDesignMethodMeanGreedy}{0.992}
\newcommand{\numDesignMethodMedianGreedy}{1.000}
\newcommand{\numDesignMethodMinSwap}{0.987}
\newcommand{\numDesignMethodMeanSwap}{0.999}
\newcommand{\numDesignMethodMedianSwap}{1.000}

\newcommand{\numDesignEvalFracStat}{0.23}

\newcommand{\numDesignAwareMinHetero}{1.000}
\newcommand{\numDesignGreedyMinHetero}{0.912}

\newcommand{\numDesignCrossHetero}{0.765}
\newcommand{\numDesignEvalFracHetero}{0.09}

\newcommand{\numDesignAwareMinRecency}{0.987}

\newcommand{\numDesignAwareMinMatern}{0.994}

\newcommand{\numSleepRecordings}{197}
\newcommand{\numSleepSubjects}{100}
\newcommand{\numSleepHours}{3818}

\newcommand{\numSleepMedian}{22.6}

\newcommand{\numLtafRecords}{84}
\newcommand{\numLtafHours}{1934}

\newcommand{\numSleepJaccardMin}{0.061}
\newcommand{\numSleepJaccardMax}{0.091}

\newcommand{\numLtafCohortAll}{84}

\newcommand{\numLtafMedianHours}{24.0}

\newcommand{\numCfContigFour}{+0.648}
\newcommand{\numCfUniformFour}{+0.682}

\newcommand{\numCfContigSixteen}{+0.659}
\newcommand{\numCfUniformSixteen}{+0.738}
\newcommand{\numCfAwareSixteen}{+0.694}
\newcommand{\numCfKqSixteen}{+0.650}
\newcommand{\numCfContigSixtyFour}{+0.863}
\newcommand{\numCfUniformSixtyFour}{+0.881}

\newcommand{\numCfAfContigOneHour}{+0.696}
\newcommand{\numCfAfDispOneHour}{+0.971}
\newcommand{\numCfAfContigFourHour}{+0.851}
\newcommand{\numCfAfDispFourHour}{+0.998}

\newcommand{\numAfObservedPctFour}{4.17\%}
\newcommand{\numAfEquivalentHoursFour}{1}
\newcommand{\numAfObservedPctSixteen}{16.67\%}

\newcommand{\numSleepFixedRangeResamples}{2000}

\newcommand{\numSleepAdjustedPositiveRanges}{9}

\newcommand{\numSleepCohortCells}{30}
\newcommand{\numSleepCohortPositiveRanges}{10}
\newcommand{\numSleepCohortNegativeRanges}{5}
\newcommand{\numSleepCohortUnresolvedRanges}{15}

\newcommand{\numSleepFullRemSixteenAdjusted}{+0.079 [-0.010, +0.168]}
\newcommand{\numSleepFullRemSixteenSC}{+0.221 [-0.010, +0.428]}
\newcommand{\numSleepFullRemSixteenST}{+0.161 [-0.445, +0.650]}

\newcommand{\numSleepFullNThreeSixteenAdjusted}{+0.151 [+0.058, +0.290]}
\newcommand{\numSleepFullNThreeSixteenSC}{-0.293 [-0.430, -0.155]}
\newcommand{\numSleepFullNThreeSixteenST}{+0.554 [+0.286, +0.900]}

\newcommand{\numSleepFullWakeSixteenAdjusted}{+0.019 [+0.005, +0.040]}
\newcommand{\numSleepFullWakeSixteenSC}{+0.321 [+0.088, +0.529]}
\newcommand{\numSleepFullWakeSixteenST}{-0.123 [-0.624, +0.890]}

\newcommand{\numAfFixedRangeResamples}{2000}
\newcommand{\numAfBootDiffOne}{-0.061}
\newcommand{\numAfBootRangeOne}{[-0.299, +0.164]}
\newcommand{\numAfBootDiffFour}{+0.274}
\newcommand{\numAfBootRangeFour}{[+0.143, +0.458]}
\newcommand{\numAfBootDiffSixteen}{+0.148}
\newcommand{\numAfBootRangeSixteen}{[+0.076, +0.249]}

\newcommand{\numResRegretCoarseSmall}{0.0510}
\newcommand{\numResRegretFineSmall}{0.03749}
\newcommand{\numResRegretCoarseLarge}{0.03706}
\newcommand{\numResEpsFourSmall}{0.2255}
\newcommand{\numResEpsFourLarge}{0.0371}

\newcommand{\numResClassSizeOne}{2}
\newcommand{\numResClassSizeTwo}{5}
\newcommand{\numResClassSizeThree}{75}
\newcommand{\numResClassSizeFour}{568}

\newcommand{\numLtafResolutionSpread}{0.0073}
\newcommand{\numLtafResolutionRcond}{$10^{-10}$}

\newcommand{\numFloorFineJaccardMin}{0.33}

\newcommand{\numFloorCoarseSpread}{0.0074}

\newcommand{\numBalanceFineJaccardMin}{0.14}

\newcommand{\numAfCohortNStrict}{71}
\newcommand{\numAfCohortStrictContigFour}{+0.596}
\newcommand{\numAfCohortStrictDispFour}{+0.958}

\newcommand{\numAfCohortStrictContigSixteen}{+0.748}
\newcommand{\numAfCohortStrictDispSixteen}{+0.998}

\newcommand{\numAfCohortNAll}{84}
\newcommand{\numAfCohortAllContigFour}{+0.696}
\newcommand{\numAfCohortAllDispFour}{+0.971}
\newcommand{\numAfCohortAllContigSixteen}{+0.851}
\newcommand{\numAfCohortAllDispSixteen}{+0.998}
\newcommand{\numAfCohortNMixed}{33}
\newcommand{\numAfCohortMixedContigFour}{+0.180}
\newcommand{\numAfCohortMixedDispFour}{+0.829}
\newcommand{\numAfCohortMixedContigSixteen}{+0.430}
\newcommand{\numAfCohortMixedDispSixteen}{+0.990}

\title{Counterfactual Evaluation of Temporal Observation Protocols}

\author{\name Xizhe Zhang \email zhangxizhe@gmail.com \\
        \addr ORCID: \href{https://orcid.org/0000-0002-8684-4591}%
        {0000-0002-8684-4591}}

\ShortHeadings{Counterfactual Evaluation of Temporal Protocols}{Zhang}

\begin{document}
\maketitle
\thispagestyle{plain}

\begin{abstract}%
We study counterfactual protocol evaluation: whether data collected under a realised observation protocol determine the predictive value of alternatives that were never deployed. Protocol value is the population $R^2$ of the Bayes-optimal predictor of a fixed trajectory-level target from the measurements an alternative would collect. We show that even infinite benchmark data need not determine this value: distinct latent covariance structures can induce the same benchmark measurement--target law while assigning different values to the same alternative. We develop a value-specific identification theory in which only latent ambiguity that changes the alternative's value matters. For linear targets, invisible covariance directions certify non-identification, while targeted measurements can restore identification without recovering the full latent covariance; an exact permutation construction extends the result to nonlinear aggregate targets. With finite dense calibration data, uniform error bounds control protocol-selection regret and distinguishable value gaps. Exact marginal gains then support cost-constrained, target-aware observation design. Simulations and retrospective analyses of Sleep-EDF and Long-Term AF show that broad temporal-layout differences can be more reliably distinguished than fine placements selected from finite data. Together, these results connect identification, calibration resolution and observation design for undeployed protocols.
\end{abstract}


\section{Introduction}
\label{sec:introduction}

Predictive performance in temporal learning depends not only on the learner, but also on the observation protocol under which the data are collected. Standard benchmarks typically hold this protocol fixed and compare learners on the resulting data. This leaves a different question unanswered: what predictive performance could be achieved under an undeployed observation protocol? We refer to this problem as \emph{counterfactual protocol evaluation}. The central question is whether the joint measurement–target distribution under the realised protocol determines the predictive value of an undeployed protocol.

This question is particularly important when the prediction target is defined over a substantially longer time horizon than the observed input. In such settings, the realised protocol may omit temporal information relevant to the target. Examples include predicting whole-night sleep-stage proportions from partially observed sleep records and long-horizon atrial-fibrillation burden from shorter monitoring windows. Weak predictive performance may therefore reflect limitations of the learner, limited sample size, or target-relevant information excluded by the observation protocol. Better models and larger datasets can address the first two limitations, but the third requires changing what is observed.

Several lines of work address related questions. Methods for incomplete or irregular temporal data learn prediction rules from irregularly or incompletely observed training data \citep{che2018grud,rubanova2019latent,kidger2020neural}. Experimental-design and sensor-placement methods, including optimal design for longitudinal data, choose future measurements to optimise a criterion under an assumed or estimated model \citep{chaloner1995bayesian,krause2008near,ji2017optimal}. Bayes-error analyses characterise predictive limits for a given feature--target distribution \citep{fukunaga1987bayes,ishida2023performance}. Most closely related, active feature acquisition performance evaluation studies how to evaluate a new acquisition policy from data collected under a different retrospective acquisition process, under identification assumptions that include sufficient support for the acquisitions required by the target policy \citep{vonkleist2025evaluation,precup2000eligibility,kallus2020double}. In our setting, an alternative protocol may instead require measurements that were never collected under the realised protocol. This leaves a different identification question: whether the joint measurement--target distribution generated by one realised observation protocol determines the predictive value of another protocol whose measurements were never collected.

We develop a theory of counterfactual protocol evaluation centred on the information needed to determine the value of an undeployed protocol. The value of an alternative protocol is identified precisely when it is constant across all latent models that induce the same realised measurement--target law \citep{koopmans1950identification,rothenberg1971identification}. This criterion is value-specific: the value may be identified even when the latent covariance is not, provided that the remaining ambiguity does not change the value being evaluated. Under a Gaussian latent-trajectory model with noisy linear measurements, we show that this condition can fail even with infinite data under the realised protocol, because distinct latent covariances can generate the same joint law of the realised measurements and target while assigning different values to the same alternative. For linear targets, we characterise the unresolved covariance variation through perturbations invisible to the realised law; nonzero directional sensitivity of the alternative's value along such a perturbation certifies local non-identification. Targeted augmentation can remove the value-changing ambiguity and identify the value of a specified alternative without recovering the full latent covariance. A stationary construction gives the sharp minimal failure case in the stationary, standardised, one-observation setting with a mean target, while a permutation construction establishes exact non-identification for nonlinear aggregate targets.

Targeted augmentation may be sufficient when the value of a specified alternative is sought. For an entire candidate family, densely observed calibration trajectories provide a stronger, family-wide route: at the population level, their law identifies the latent covariance used to evaluate all candidates. With only finitely many such trajectories, two questions remain: how accurately candidate values can be distinguished and how those distinctions should guide design. We derive uniform bounds that propagate covariance-estimation error to protocol-value error over the candidate family. The resulting bounds yield selection-regret guarantees and a calibration resolution below which value differences cannot be distinguished reliably. They also link the useful granularity of the candidate class to the amount of calibration data: finer classes can improve the attainable value, but their increasingly small differences support reliable selection only when the data can resolve them. At that resolution, we formulate cost-constrained, target-aware observation design and derive exact rank-one marginal gains for adding measurements under constraints on timing, support, repetition, noise, and cost. Simulations examine calibration error, protocol selection, and design under known data-generating laws. Retrospective analyses of Sleep-EDF and Long-Term AF annotations compare alternative temporal layouts at matched observation budgets \citep{kemp2000analysis,mourtazaev1995age,goldberger2000physiobank,petrutiu2007abrupt}; the pooled analyses show clearer differences between pre-specified dispersed and contiguous layouts than among learned Sleep supports, whose locations and held-out advantages vary across targets and resamples.

Taken together, these results place identification before optimisation in temporal observation design. The information required to evaluate a protocol is not generally the information required to recover the full latent dependence structure; it is the information needed to eliminate ambiguity that changes the protocol values under consideration. Finite calibration then determines the resolution at which those values can be compared, and the granularity of the design problem should be matched to that resolution. This value-directed view links the choice of additional measurements, the amount of calibration data, and the level of detail at which candidate protocols can be compared and optimised.

The remainder of the paper is organised as follows. Section 2 defines temporal aggregate targets, observation protocols, and protocol value. Section 3 develops value-specific identification and targeted augmentation. Section 4 studies finite calibration and protocol comparison, and Section 5 develops target-aware observation design. Section 6 reports the simulations and retrospective analyses. Section 7 reviews related work, and Sections 8 and 9 present the discussion and conclusion.

\section{Problem Formulation and Protocol Value}
\label{sec:formulation}
Consider a benchmark collected under one observation protocol and a future study
that may use another. The benchmark contains the measurements produced by the
realised protocol together with a trajectory-level target used for prediction.
The full latent trajectory need not be observed. An alternative protocol may
include measurements that were not collected in the benchmark. We first define
the target and protocol measurements in continuous time and then introduce the
discrete Gaussian model used in the theoretical analysis.

An \textit{observation protocol} is a finite collection of acquisition actions. An \textit{observation design} is the decision problem of choosing a protocol from a feasible family. The \textit{realised protocol} is the one that generated the benchmark data; an \textit{alternative protocol} is a feasible protocol whose predictive value is to be evaluated even though the benchmark does not contain the complete measurement vector that it would produce.

\subsection{Temporal Aggregate Targets and Observation Protocols}
\label{sec:temporal-protocol}

For unit $i$, let $Z_i=\{Z_i(t):t\in[0,T]\}$ be a latent scalar trajectory. The
supervised target is one number obtained from the whole trajectory. Specifically,
we first transform the state at each time by a function $g$ and then average the
transformed values over the observation horizon:
\begin{equation}
\Theta_{i,g}=\int_0^T\omega_T(t)\,g\{Z_i(t)\}\dd t,
\qquad \omega_T(t)\ge0,
\qquad \int_0^T\omega_T(t)\dd t=1 .
\label{eq:label-continuous}
\end{equation}
The weight function $\omega_T$ specifies which parts of the horizon contribute
to the target. Uniform weights, $\omega_T(t)=1/T$, give an ordinary time
average; non-uniform weights can emphasise particular periods.

Two choices of $g$ recur throughout the paper. If $g(z)=z$, then
$\Theta_{i,g}$ is the average level of the latent process, such as average
electricity demand over a day. If
$g_c(z)=\Ind\{z>c\}$, then $\Theta_{i,g_c}$ is the weighted fraction of time for
which the process exceeds $c$, and therefore represents exceedance-time burden.
Sleep-stage proportions, atrial-fibrillation burden and time spent in a
symptomatic state are examples of occupation-type targets.
The target is fixed when protocols are compared: changing the protocol changes
what is measured, not the quantity to be predicted.

One acquisition action $a$ returns a noisy linear measurement of the
trajectory,
$Y_{i,a}=L_aZ_i+\varepsilon_{i,a}$, where $L_a$ may represent a point
measurement or an average over a time window. We assume
$\varepsilon_{i,a}\sim\cN(0,r_a)$, independently across actions and units and
independently of $Z_i$. Each action has cost $c_a$. For a cost budget
$\mathcal B>0$, a protocol is a finite set $S=\{a_1,\ldots,a_D\}$ satisfying
$\sum_{a\in S}c_a\le\mathcal B$, and $\Pi_{\mathcal B}$ denotes the collection
of all such feasible protocols. Identification and estimation use this action
representation; \cref{sec:design} specifies how actions encode timing, support,
repetition, noise and cost.

\subsection{Discrete Gaussian Model}
\label{sec:discrete}

The theoretical analysis uses a grid $t_1<\cdots<t_p$ on $[0,T]$. We collect
the latent states on this grid in
$Z_i=(Z_i(t_1),\ldots,Z_i(t_p))^\top$ and assume
\begin{equation}
Z_i\sim\cN(0,K),\qquad \diag(K)=\Ind_p.
\label{eq:discrete-model}
\end{equation}
Thus $K$ records the dependence between times, while every marginal state has
mean zero and variance one. This standardisation fixes the scale on which $g$
is defined. Here $\diag(M)$ is the vector of diagonal entries of $M$, whereas
$\Diag(v)$ is the diagonal matrix with diagonal $v$.

Let $\omega=(\omega_1,\ldots,\omega_p)^\top$ be non-negative quadrature weights
with $\Ind_p^\top\omega=1$. The continuous target and the protocol measurements
then become
\begin{equation}
\Theta_{i,g}=\sum_{j=1}^p\omega_j\,g(Z_{ij}),
\qquad
Y_{i,S}=A_SZ_i+\varepsilon_{i,S},
\qquad
\varepsilon_{i,S}\sim\cN(0,R_S).
\label{eq:discrete-protocol}
\end{equation}
Each row $\ell_a^\top$ of $A_S\in\R^{D\times p}$ describes one action. A point
measurement at $t_j$ has $\ell_a=e_j$. A window measurement has non-negative
entries on the grid points covered by that window and is normalised so that
$\Ind_p^\top\ell_a=1$. With independent action noise,
$R_S=\Diag(r_a)_{a\in S}$. In short, $K$ describes the latent trajectory,
$\omega$ describes the target, and $(A_S,R_S)$ describes what protocol $S$
observes.

The empty protocol has no observations and explains no target variation. We
set
\begin{equation}
Q_\varnothing(K)=0_{p\times p},\qquad \cI_g(\varnothing;K)=0.
\label{eq:empty-protocol}
\end{equation}

\subsection{Protocol Value under the Gaussian Model}
\label{sec:protocol-values}

We value a protocol by the fraction of target variance explained by the best
population predictor based on its measurements.

\begin{definition}[Bayes value of an observation protocol]
\label{def:values}
Let $\Theta\in L^2$ be a scalar target with $\Var(\Theta)>0$. The Bayes protocol
value under squared loss is
\begin{equation}
\cI_{\mathrm{Bayes}}(S)=\frac{\Var\{\E(\Theta\mid Y_S)\}}{\Var(\Theta)}.
\label{eq:bayes-value}
\end{equation}
It is the population $R^2$ of the optimal predictor based on $Y_S$ and lies in
$[0,1]$.
\end{definition}

The best-linear analogue restricts the predictor to be affine. With
$\Sigma_S=\Var(Y_S)$ and $c_S=\Cov(Y_S,\Theta)$, it is
\begin{equation}
\cI_{\mathrm L}(S)=\frac{c_S^\top\Sigma_S^{+}c_S}{\Var(\Theta)},
\label{eq:linear-value}
\end{equation}
where $\Sigma_S^{+}$ is the Moore--Penrose inverse
\citep{penrose1955generalized}.
It depends only on second moments and satisfies
$\cI_{\mathrm L}(S)\le\cI_{\mathrm{Bayes}}(S)$, with equality when
$\E(\Theta\mid Y_S)$ is affine in $Y_S$.

Equations~\eqref{eq:bayes-value} and \eqref{eq:linear-value} define population
quantities. The empirical analyses evaluate fitted predictors by pooled
held-out cross-fitted $R^2$.

Under the discrete Gaussian model, the Bayes value can be calculated from $K$,
$g$ and the protocol matrices. The target transformation $g$ converts a latent
Gaussian correlation into a covariance as follows. For standard bivariate
normal variables
$(U,V_r)$ with correlation $r$, define
\begin{equation}
C_g(r)=\Cov\{g(U),g(V_r)\}
=\sum_{k\ge1}\frac{a_k^2}{k!}\,r^k,
\qquad g\in L^2(\phi),
\label{eq:Cg}
\end{equation}
where $\phi$ is the standard normal distribution and
$g-\E g=\sum_{k\ge1}(a_k/k!)H_k$ is the Hermite expansion of $g$ in
$L^2(\phi)$. In particular,
$\Cov\{g(Z_j),g(Z_k)\}=C_g(K_{jk})$. All dependence on the possibly nonlinear
target function enters the value calculation through $C_g$.

The protocol enters through the part of the latent trajectory recoverable from
its measurements. Let $\Sigma_S(K)=A_SKA_S^\top+R_S$. For a non-empty protocol
with nonsingular $\Sigma_S(K)$, Gaussian conditioning gives
\begin{equation}
Q_S(K)=KA_S^\top\Sigma_S(K)^{-1}A_SK
      =\Cov\{\E(Z\mid Y_S)\}.
\label{eq:QP}
\end{equation}
Thus $Q_S(K)$ is the covariance of the part of the trajectory recovered from
$Y_S$. The residual covariance used in the marginal-gain calculation is
introduced in \cref{sec:design}.

Combining this covariance transform with Gaussian conditioning yields the
protocol-value identity below \citep{zhang2026label}.

\begin{proposition}[Gaussian protocol-value identity; Zhang, 2026]
\label{prop:risk-identity}
Under \eqref{eq:discrete-model}--\eqref{eq:discrete-protocol}, let
$g\in L^2(\phi)$ and assume $V_g(K)>0$. Define
\begin{equation}
V_g(K)=\sum_{j,k=1}^p\omega_j\omega_k\,C_g(K_{jk}),
\qquad
F_g(S;K)=\sum_{j,k=1}^p\omega_j\omega_k\,C_g\{Q_S(K)_{jk}\}.
\label{eq:VF}
\end{equation}
Then $V_g(K)=\Var(\Theta_g)$ and
$F_g(S;K)=\Var\{\E(\Theta_g\mid Y_S)\}$, and
\begin{equation}
\cI_g(S;K)=\cI_{\mathrm{Bayes}}(S)=\frac{F_g(S;K)}{V_g(K)}.
\label{eq:ceiling}
\end{equation}
\end{proposition}

Thus $V_g(K)$ is total target variance, $F_g(S;K)$ is the part explained by
protocol $S$, and $\cI_g(S;K)$ is the corresponding Bayes-optimal population
$R^2$ based on $Y_S$. A proof is given in
Appendix~\ref{sec:app-identification}.

For two target functions, the covariance transform has a simple form. For the
mean target $g(z)=z$, $C_g(r)=r$, so
$V_g(K)=\omega^\top K\omega$ and
$F_g(S;K)=\omega^\top Q_S(K)\omega$; equivalently,
$\Theta=\omega^\top Z$. For the occupation target
$g_c(z)=\Ind\{z>c\}$, Plackett's identity \citep{plackett1954reduction} gives
\begin{equation}
C_{g_c}(r)=G_c(r)
=\int_0^r\frac{\exp\{-c^2/(1+s)\}}{2\pi\sqrt{1-s^2}}\dd s,
\label{eq:plackett}
\end{equation}
with $G_0(r)=\arcsin(r)/(2\pi)$.

\section{Value-Specific Identification}
\label{sec:identifiability}

Section~\ref{sec:formulation} showed how to calculate protocol value when the
latent covariance $K$ is known. We now ask what can be learned when the only
available population information is the joint law of the measurements and
target under the realised protocol $A$. Throughout this section, $A$ and $B$
denote the measurement operators induced by the realised and alternative
protocols.
Sections~\ref{sec:estimation}--\ref{sec:design} return to the action-set notation
$S$, with measurement matrix $A_S$.

We begin with the linear target $\Theta=h^\top Z$. Then $(Y_A,\Theta)$ is jointly
Gaussian, so its full law is determined by second moments. We abbreviate
$\cI(B;K)=\cI_g(B;K)$ for
$g(z)=z$; in this linear-Gaussian case it also equals $\cI_{\mathrm L}(B)$.
We index grid points by $0,\dots,p-1$. Nonlinear aggregates require more than
moment matching; we treat them separately through an exact permutation argument
in \cref{prop:permutation}.

The law of $(Y_A,\Theta)$ may fail to identify $K$ even when it determines the
scalar value of the specified alternative $B$.

\subsection{Identification Criterion}
\label{sec:value-specific-identification}

For a benchmark law $\Prob_A$, define its observational equivalence class by
\begin{equation}
\mathcal K_A(\Prob_A)=\{K:\Prob_K(Y_A,\Theta)=\Prob_A\}.
\label{eq:equivalence-class}
\end{equation}
Every $K\in\mathcal K_A(\Prob_A)$ generates exactly the same distribution of
the benchmark data. No amount of additional sampling under the same protocol
$A$ can distinguish them. Consequently, for any target and model class under
consideration,
\begin{equation}
\boxed{\begin{gathered}
\cI_g(B;\cdot)\ \text{is identified from }\Prob_A\\[-1pt]
\Longleftrightarrow\quad
\cI_g(B;K_1)=\cI_g(B;K_2)
\quad\text{for all }K_1,K_2\in\mathcal K_A(\Prob_A).
\end{gathered}}
\label{eq:value-specific-criterion}
\end{equation}
This criterion is value-specific: it requires uniqueness only of the specified
alternative's value, not of every component of $K$. The class
$\mathcal K_A(\Prob_A)$ may contain many covariance structures while assigning
a unique value to $B$. Conversely, two members of this class with different
values establish non-identification. This is functional identification over an
observational-equivalence class in the sense of
\citet{koopmans1950identification}.

For the linear target $\Theta=h^\top Z$, the realised benchmark law reveals
three covariance blocks. Take $h=\omega$, and let
$Y_A=AZ+\varepsilon_A$ with $\varepsilon_A\sim\cN(0,R_A)$ and $R_A$ known.
Then $(Y_A,\Theta)$ is jointly Gaussian with covariance
\begin{equation}
\begin{pmatrix}
AKA^\top+R_A & AKh\\
h^\top KA^\top & h^\top Kh
\end{pmatrix}.
\label{eq:observables}
\end{equation}
The Gaussian benchmark law therefore depends on $K$ through the three blocks
\begin{equation}
AKA^\top,\qquad AKh,\qquad h^\top Kh,
\label{eq:three-blocks}
\end{equation}
which constitute its complete covariance information about $K$. The first
block is the covariance of the realised measurements after subtracting the
known measurement noise. The second is their cross-covariance with the target,
and the third is the target variance. Any change in $K$ that leaves all three
blocks unchanged is invisible to the benchmark.

\subsection{Minimal Stationary Counterexample}

We first ask when invisible dependence can arise in the simplest stationary
setting. A one-point benchmark reveals no lag information through
$\Var(Y_A)$, because the process is standardised. It reveals only two weighted
sums of the lag correlations: one through
$\Cov(Y_A,\Theta)$ and one through $\Var(\Theta)$. These two equations determine
all unknown lags on grids of at most three points. On a four-point grid they
leave one degree of freedom for the first time, and moving along that direction
can change the value of a protocol that observes $(Z_1,Z_2)$.

\begin{theorem}[Sharp minimal stationary counterexample]
\label{thm:minimal}
Let $p\ge2$ and let $Z$ be stationary and standardised on the grid
$\{0,1,\dots,p-1\}$ with
correlations $\rho(1),\dots,\rho(p-1)$, let $\Theta=p^{-1}\sum_j Z_j$, and let
$A$ observe $Z_0$ with known noise variance $\nu^2$. Then
\begin{enumerate}[label=(\roman*),leftmargin=2.2em,itemsep=2pt]
\item for $p\le3$ the map from $(\rho(1),\dots,\rho(p-1))$ to the observable
functionals is injective, so the stationary lag-correlation vector, and hence
$K$ within this model class, is identified;
\item for $p\ge4$ the space of stationary invisible directions has dimension
$p-3\ge1$.
\end{enumerate}
For $p=4$, the kernel is spanned by the lag perturbation
\begin{equation}
\delta=(\delta(0),\delta(1),\delta(2),\delta(3))=(0,1,-2,1).
\label{eq:delta4}
\end{equation}
For any positive-definite base correlation matrix $K$ with lag vector $\rho$,
let $K_\pm$ have lag vectors $\rho\pm\varepsilon\delta$, with lag $0$ fixed at
one. For all sufficiently small $\varepsilon>0$, these matrices are admissible
and observationally equivalent under $A$. If $B$ observes $(Z_1,Z_2)$ with
independent measurement noise $R_B=\nu_B^2 I_2$, where $\nu_B^2\ge0$, then
\begin{equation}
\cI(B;K_\pm)=\frac{2b^2}{\{1+\nu_B^2+\rho(1)\pm\varepsilon\}\,\Var(\Theta)},
\qquad
b=\frac{1+2\rho(1)+\rho(2)}{4},
\label{eq:gap-closed-form}
\end{equation}
so the two values differ whenever $1+2\rho(1)+\rho(2)\neq0$.
\end{theorem}

\begin{proof}[Proof idea]
The realised point and the target reveal two independent weighted sums of the
$p-1$ unknown lag correlations. Their Jacobian has rank
$\min\{2,p-1\}$, leaving $p-3$ free directions exactly when $p\ge4$; at
$p=4$, solving the two equations gives $(1,-2,1)$. Along this direction the
target variance and its covariance with $Z_1$ and $Z_2$ remain fixed, whereas
the covariance of $(Z_1,Z_2)$ changes, which yields
\eqref{eq:gap-closed-form}. The calculation is given in
Appendix~\ref{sec:app-identification}.
\end{proof}

This four-point threshold is specific to the stationary, standardised,
one-observation model with a uniform mean target. \Cref{fig:identifiability}
uses $p=4$, $h=\tfrac{1}{4}\mathbf{1}_4$, the baseline profile
$\rho_0(u)=\exp(-u)$ (that is, $\tau=1$), and the raw lag perturbation
$\varepsilon=\numIdEps$ in \eqref{eq:delta4}. Both protocols are noiseless,
with $A=e_0^\top$ and $B=(e_1,e_2)^\top$. The two Toeplitz correlation
matrices $K_\pm$ preserve $\Var(Y_A)=1$,
$\Cov(Y_A,\Theta)=\numIdCovYTheta$ and
$\Var(\Theta)=\numIdVarTheta$ to a maximum discrepancy of
\numIdDiscrepancy. Substitution into \eqref{eq:gap-closed-form} gives
$\cI(B;K_+)=\numIdCeilPlus$ and $\cI(B;K_-)=\numIdCeilMinus$.

\begin{figure}[t]
\centering
\includegraphics[width=0.72\linewidth]{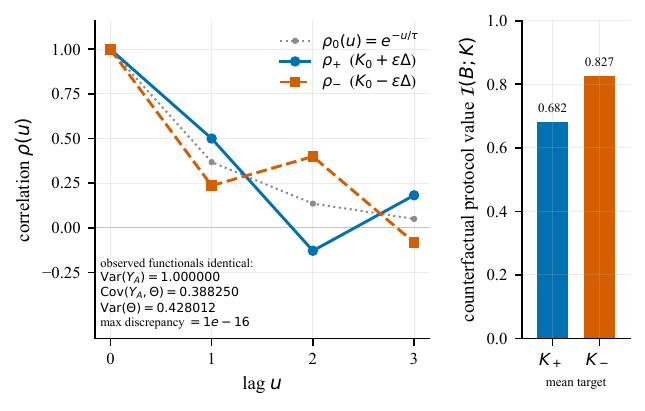}
\caption{Four-point instance of \cref{thm:minimal}. The profiles
$\rho_\pm=\rho_0\pm\varepsilon\delta$ are observationally equivalent under
$A=\{Z_0\}$ but assign values \numIdCeilPlus{} and \numIdCeilMinus{} to
$B=\{Z_1,Z_2\}$.}
\label{fig:identifiability}
\end{figure}

\FloatBarrier

\subsection{Invisible Directions}

The vector $(0,1,-2,1)$ is one instance of a more general object: a covariance
perturbation that changes latent dependence while preserving every covariance
block visible under $A$.

\begin{definition}[Invisible direction]
\label{def:invisible}
A symmetric $\Delta\in\R^{p\times p}$ with $\diag(\Delta)=0$ is
\emph{invisible} to $(A,h)$ if
\begin{equation}
A\Delta A^\top=0,\qquad A\Delta h=0,\qquad h^\top\Delta h=0 .
\label{eq:invisible}
\end{equation}
The three equalities preserve, respectively, the realised-measurement
covariance, the realised-measurement--target cross-covariance and the target
variance. The zero-diagonal requirement preserves the unit diagonal. For
positive-definite $K$, sufficiently small perturbations $K\pm\varepsilon\Delta$ remain
correlation matrices. We write $\cD(A,h)$ for the
linear space of directions invisible to $(A,h)$.
\end{definition}

\begin{theorem}[Value-changing invisible directions imply non-identification]
\label{thm:impossibility}
Let $0\neq\Delta\in\cD(A,h)$ and let $K_0\succ0$ be a correlation matrix. For
all sufficiently small $|\varepsilon|$, $K_\varepsilon=K_0+\varepsilon\Delta$
is a correlation matrix and has the same benchmark law as $K_0$; in
particular,
\begin{equation}
\Prob_{K_\varepsilon}(Y_A,\Theta)=\Prob_{K_0}(Y_A,\Theta).
\label{eq:obs-equivalent}
\end{equation}
Let $B$ be an alternative protocol with noise $R_B\succeq0$ such that
$\Sigma_B(K_0)=BK_0B^\top+R_B\succ0$, and set
$W=\Sigma_B(K_0)^{-1}$. Its directional derivative is
\begin{equation}
\begin{aligned}
D_B(\Delta;K_0)
&:=\left.\frac{\dd}{\dd\varepsilon}
\cI(B;K_0+\varepsilon\Delta)\right|_{\varepsilon=0}\\
&=\frac{2h^\top\Delta B^\top WBK_0h
-h^\top K_0B^\top W(B\Delta B^\top)WBK_0h}{h^\top K_0h}.
\end{aligned}
\label{eq:ceiling-derivative}
\end{equation}
If $D_B(\Delta;K_0)\neq0$, then the value functional $\cI(B;\cdot)$ is not
locally identified at $K_0$ within the model class: for every sufficiently
small $\varepsilon>0$, the observationally equivalent matrices
$K_0\pm\varepsilon\Delta$ assign different values to $B$.
\end{theorem}

\begin{proof}[Proof idea]
The three invisibility equations preserve every covariance block of the
Gaussian benchmark law. Differentiating the alternative-protocol value with
the inverse rule gives \eqref{eq:ceiling-derivative}; a non-zero derivative
then separates the two compatible paths $K_0\pm\varepsilon\Delta$ to first
order. The full matrix calculation appears in
Appendix~\ref{sec:app-identification}.
\end{proof}

Thus no estimator based only on $(Y_A,\Theta)$ can be consistent at both local
alternatives, and collecting more units under the same protocol does not
resolve the ambiguity.

If $A$ has $d$ rows, rank--nullity gives a lower bound on the dimension of the
invisible space. A standardised $K$ has $p(p-1)/2$ free entries whereas
\eqref{eq:three-blocks} supplies at most $d(d+1)/2+d+1$ linear constraints.
Rank--nullity therefore gives
\begin{equation}
\dim\cD(A,h)\ge
\frac{p(p-1)}2-\frac{d(d+1)}2-d-1,
\qquad
\text{when the right-hand side is positive.}
\label{eq:counting}
\end{equation}
Near a positive-definite $K_0$, the equivalence class contains the corresponding
open piece of the affine space $K_0+\cD(A,h)$
\citep{rothenberg1971identification}. This count only measures ambiguity in
$K$: value non-identification still requires sensitivity in
\eqref{eq:ceiling-derivative}. With fixed $d$, the contrast is nevertheless
stark---the benchmark contributes $O(d^2)$ constraints to $O(p^2)$ covariance
parameters.

\subsection{Nonlinear Targets}

The invisible-direction theorem uses joint Gaussianity of
$(Y_A,\Theta)$ and therefore applies directly to a linear target. With a
nonlinear aggregate, equality of the first two moments does not imply equality
of the full benchmark law. The following symmetry argument instead constructs
two exactly equivalent benchmark laws. The condition $A\Xi=A$ says that the
permutation leaves every realised measurement unchanged, while
$\Xi^\top\omega=\omega$ says that it preserves the aggregate weights.

\begin{proposition}[Permutation equivalence for arbitrary targets]
\label{prop:permutation}
Let $\Xi$ be a permutation matrix with $A\Xi=A$ and $\Xi^\top\omega=\omega$, and put
$K'=\Xi K\Xi^\top$. For every $g\in L^2(\phi)$ with $V_g(K)>0$, the joint laws of
$(Y_A,\Theta_g)$ under $K$ and under $K'$ coincide, and
\begin{equation}
\cI_g(B;K')=\cI_g(B\Xi;K)
\qquad\text{for every alternative protocol }B .
\label{eq:perm-transport}
\end{equation}
Consequently, whenever $\cI_g(B\Xi;K)\neq\cI_g(B;K)$, the value of $B$ is not
identified from the benchmark law.
\end{proposition}

Here $B\Xi$ is the algebraically permuted measurement operator used to compare
values, with the same observation-noise covariance $R_B$ as $B$; it need not
belong to the original candidate family.

\begin{proof}
Let $Z\sim\cN(0,K)$ and $Z'=\Xi Z\sim\cN(0,K')$. Then $AZ'=A\Xi Z=AZ$, and since
$g$ acts entrywise $g(\Xi Z)=\Xi g(Z)$, so
$\Theta_g(Z')=\omega^\top\Xi g(Z)=(\Xi^\top\omega)^\top g(Z)=\Theta_g(Z)$; in
particular $V_g(K')=V_g(K)>0$. The measurement noise is independent with the
same law under both, so the joint laws agree. Applying the same substitution to
$B$ gives \eqref{eq:perm-transport}.
\end{proof}

Non-identification follows whenever the permutation changes the value of $B$.
For example, no change occurs when $\Xi K\Xi^\top=K$. The next construction
shows that three grid points suffice for every nonconstant target under uniform
aggregation.

\begin{example}[Three points suffice under uniform aggregation]
\label{ex:three-point}
Take $p=3$, $A=e_0^\top$, uniform $\omega$, let $\Xi$ transpose coordinates $1$
and $2$, and let $B=e_1^\top$ observe $Z_1$ with noise $\nu_B^2$, writing
$s=1+\nu_B^2$. For $0<a<1$ put
\begin{equation}
K_a=\begin{pmatrix}1&a&0\\a&1&0\\0&0&1\end{pmatrix},
\qquad K_a'=\Xi K_a\Xi^\top .
\label{eq:three-point}
\end{equation}
Both are correlation matrices, $A\Xi=A$ and $\Xi^\top\omega=\omega$, so the
benchmark laws coincide. Nevertheless,
$\cI_g(B;K_a)>\cI_g(B;K_a')$ for every nonconstant
$g\in L^2(\phi)$, every $a\in(0,1)$ and every noise level. The explicit
calculation, which reduces the gap to positive values of $C_g$, is given in
Appendix~\ref{sec:app-identification}.
\end{example}

Without stationarity, three points suffice for arbitrary nonconstant
square-integrable transformations under uniform aggregation; within the
stationary standardised model with a one-point benchmark and a mean target, the
sharp minimum is four.

\subsection{Identification by Augmentation}
\label{sec:augmentation}

Suppose $A$ has already been run and a further acquisition $A'$ is contemplated
on the same units. Let $A_{\rm aug}$ stack the rows of $A$ and $A'$. Their joint
measurement reveals the covariance within $A'$, its cross-covariance with $A$,
and the corresponding target cross-covariances. Collect the revealed linear
directions in
\begin{equation}
\bar A=\begin{bmatrix}A_{\rm aug}\\ h^\top\end{bmatrix} .
\label{eq:abar}
\end{equation}
Let
\begin{equation}
\cD(A_{\rm aug},h)=\big\{\Delta=\Delta^\top:\diag(\Delta)=0,\;
A_{\rm aug}\Delta A_{\rm aug}^\top=0,\;
A_{\rm aug}\Delta h=0,\;h^\top\Delta h=0\big\}
\label{eq:joint-invisible}
\end{equation}
be the invisible space of the augmented benchmark.

\begin{proposition}[Exact sufficiency and a first-order failure certificate]
\label{prop:augmentation}
Fix $K_0\succ0$ with $\diag(K_0)=\Ind$ and an alternative protocol $B$ with
known $R_B\succeq0$ such that $BK_0B^\top+R_B\succ0$.
\begin{enumerate}[label=(\roman*),leftmargin=2.2em,itemsep=2pt]
\item \emph{Exact sufficiency.} If $\rowspace(B)\subseteq\rowspace(\bar A)$,
then the joint Gaussian law of $(Y_B,\Theta)$, and hence $\cI(B;K)$, is
identified exactly from the augmented benchmark, for every admissible $K$.
\item \emph{First-order failure certificate.} If $D_B(\Delta;K_0)\neq0$ for
some $\Delta\in\cD(A_{\rm aug},h)$, then $\cI(B;\cdot)$ is not identified near
$K_0$: for every sufficiently small $\varepsilon>0$ the correlation matrices
$K_0\pm\varepsilon\Delta$ induce the same law of $(Y_{A_{\rm aug}},\Theta)$ and
different values of $B$.
\end{enumerate}
\end{proposition}

The proof, including the row-space factorisation used in part (i), is given in
Appendix~\ref{sec:app-identification}.

By contraposition, local identification requires $D_B(\cdot;K_0)$ to vanish on
$\cD(A_{\rm aug},h)$. This condition need not be sufficient, just as the
row-space condition is sufficient but need not be necessary. Both checks are
linear algebra: compute a row-space inclusion and the kernel in
\eqref{eq:joint-invisible}, then restrict the linear functional $D_B$ to that
kernel.

The number of additional measurements required depends on the covariance
restrictions. In the four-point
stationary model of \cref{thm:minimal}, augmenting $\{Z_0\}$ with $Z_1$ reduces
the stationary invisible space from dimension $\numAugStatBefore$ to
$\numAugStatAfter$ even though the span condition for $B=\{Z_1,Z_2\}$ fails:
identification there comes from stationarity, not from
\cref{prop:augmentation}(i). Without that restriction the invisible dimensions
for $\{Z_0\}$, $\{Z_0,Z_1\}$ and $\{Z_0,Z_1,Z_2\}$ are $\numAugFreeBefore$,
$\numAugFreeOne$ and $\numAugFreeTwo$: observing $Z_1$ leaves a two-dimensional
invisible space, whereas observing both $Z_1$ and $Z_2$ removes all invisible
directions. Under stationarity, observing $Z_1$ already suffices. For an affine
structured covariance class with direction space $\cT$, the same local test
applies on $\cD(A_{\rm aug},h)\cap\cT$; stationarity is the Toeplitz instance.

A calibration protocol of full column rank satisfies the span condition and
removes every invisible direction in the unrestricted model. Full column rank
is stronger than is needed to identify one specified alternative, but it gives
a simple route to evaluating an entire candidate family. The next question is
how accurately that route works when dense calibration is available for only
$m$ units.

\section{Finite Calibration and Protocol Comparison}
\label{sec:estimation}

Calibration measurements constrain covariance directions that are invisible
under the realised protocol. We consider $m$ independent units observed densely
according to
\begin{equation}
W^{(i)}=Z^{(i)}+\eta^{(i)},\qquad
Z^{(i)}\sim\cN(0,K),\qquad \eta^{(i)}\sim\cN(0,R_0),
\qquad i=1,\dots,m.
\label{eq:calibration}
\end{equation}
The pairs $(Z^{(i)},\eta^{(i)})$ are independent across $i$, and
$\eta^{(i)}$ is independent of $Z^{(i)}$ for every $i$. Consequently,
$\Cov(W^{(i)})=K+R_0$.
The finite-sample question is how error in the covariance recovered from these
data propagates to comparisons among protocols. Appendix~\ref{sec:app-calibration}
gives the proofs, including covariance-repair and resolvent constants.

\subsection{Covariance Calibration}

Let $\widehat\Sigma=m^{-1}\sum_iW^{(i)}W^{(i)\top}$, and set
$\widehat R_0=R_0$ when the calibration-noise covariance is known;
otherwise estimate $R_0$ from an independent or replicated calibration
component. For a symmetric matrix $M$, define
\begin{equation}
\begin{aligned}
\widetilde M_\tau
&=\proj_{\{H=H^\top:\,H\succeq\tau I_p\}}(M),\\
D_\tau(M)&=\Diag\{\diag(\widetilde M_\tau)\},&
\cR_\tau(M)&=D_\tau(M)^{-1/2}\widetilde M_\tau D_\tau(M)^{-1/2}.
\end{aligned}
\label{eq:repair-operator}
\end{equation}
Thus $\cR_\tau$ floors the eigenvalues and then rescales the result to a
correlation matrix. The plug-in estimator is
\begin{equation}
\widehat K=\cR_{\tau_m}(\widehat\Sigma-\widehat R_0),
\qquad
\widehat\cI_g(S)=\frac{F_g(S;\widehat K)}{V_g(\widehat K)},
\label{eq:plugin}
\end{equation}
where $\tau_m\downarrow0$. The projection in
\eqref{eq:repair-operator} is the Frobenius-norm eigenvalue-clipping map
\citep{higham1988computing}. The asymptotic results require only
$\tau_m\downarrow0$.

The target enters the error rate through the modulus of continuity of its
Gaussian covariance transform $C_g$.

\begin{proposition}[Regularity of the target covariance transform]
\label{prop:regularity}
For each of the following cases there is a finite constant $L$ such that, on
the stated range,
\begin{equation}
|C_g(r_2)-C_g(r_1)|\le L|r_2-r_1|^\beta.
\label{eq:target-modulus}
\end{equation}
\begin{enumerate}[label=(\roman*),leftmargin=2.2em,itemsep=2pt]
\item\label{lem:lipschitz} If $g$ belongs to the first-order Gaussian Sobolev
space $W^{1,2}(\phi)$, then
\eqref{eq:target-modulus} holds on $[-1,1]$ with $\beta=1$.
\item\label{lem:holder} If $g_c(z)=\Ind\{z>c\}$, then it holds on $[-1,1]$
with $\beta=1/2$. This exponent cannot be improved uniformly, because
\begin{equation}
C_{g_c}(1)-C_{g_c}(1-\delta)
\sim \frac{e^{-c^2/2}}{2\pi}\sqrt{2\delta},
\qquad \delta\downarrow0.
\label{eq:boundary-asymptotic}
\end{equation}
\item\label{lem:local-lipschitz} For the same threshold target and every
$0\le r_{\max}<1$, \eqref{eq:target-modulus} holds on
$[-r_{\max},r_{\max}]$ with $\beta=1$.
\end{enumerate}
\end{proposition}

Smooth targets therefore transmit covariance error linearly. Threshold targets
do so linearly at a fixed model whose relevant correlations are separated from
$\pm1$, but only through a square-root envelope when boundary correlations are
allowed.

\subsection{Uniform Value Error}

We now propagate covariance-calibration error through the protocol-value
calculation,
$\widehat K-K\mapsto Q_S(\widehat K)-Q_S(K)\mapsto
\widehat\cI_g(S)-\cI_g(S)$. Two quantities can destabilise this chain and
therefore appear explicitly in the assumptions below: a vanishing target
variance and nearly singular protocol covariances.

\begin{theorem}[Uniform stability of protocol values]
\label{thm:uniform-error}
Let $V_g(K)\ge v_0>0$, and suppose that for a feasible family $\Pi_{\mathcal B}$,
\begin{equation}
\sup_{\substack{S\in\Pi_{\mathcal B}\\S\ne\varnothing}}
\opnorm{A_S}^2\le a<\infty,
\qquad
\inf_{\substack{S\in\Pi_{\mathcal B}\\S\ne\varnothing}}
\lambda_{\min}(A_SKA_S^\top+R_S)\ge\lambda>0.
\label{eq:family-conditioning}
\end{equation}
If the modulus \eqref{eq:target-modulus} with exponent $\beta$ is valid for all
arguments of $C_g$ arising from entries of $K$, $\widehat K$, $Q_S(K)$ and
$Q_S(\widehat K)$, then, for all sufficiently small
$e=\opnorm{\widehat K-K}$,
\begin{equation}
\sup_{S\in\Pi_{\mathcal B}}
\big|\widehat\cI_g(S)-\cI_g(S)\big|\le C e^\beta.
\label{eq:uniform-error}
\end{equation}
Here $C$ depends only on the target modulus, $v_0$, bounds on $K$ and
$\widehat K$, and the family bounds $(a,\lambda)$, not on $S$.
Consequently $\beta=1$ for smooth targets, $\beta=1/2$ globally for threshold
targets, and $\beta=1$ for threshold targets at an interior model.
\end{theorem}

For a threshold target, the interior condition invoked in the theorem can be
checked at the fixed model through
\begin{equation}
r_0=\max\left\{\max_{j\ne k}|K_{jk}|,\ 
\sup_{S\in\Pi_{\mathcal B}}\max_{j,k}|Q_S(K)_{jk}|\right\}<1.
\label{eq:r0}
\end{equation}
Uniform continuity of the resolvent then places the corresponding arguments
for every sufficiently close $\widehat K$ in a common compact subset of
$(-1,1)$.

\subsection{Calibration Rates}

\begin{theorem}[Fixed-model calibration rates]
\label{thm:fixed-model}
Fix $p$ and a finite feasible family $\Pi_{\mathcal B}$, and suppose that
$K\succ0$, $V_g(K)>0$ and the conditioning bound in
\eqref{eq:family-conditioning} holds. If
$\opnorm{\widehat R_0-R_0}=O_p(m^{-1/2})$ and $\tau_m\downarrow0$, then the
estimator in \eqref{eq:plugin} satisfies
$\opnorm{\widehat K-K}=O_p(m^{-1/2})$. Hence
\begin{equation}
\sup_{S\in\Pi_{\mathcal B}}|\widehat\cI_g(S)-\cI_g(S)|=
\begin{cases}
O_p(m^{-1/2}), & g\in W^{1,2}(\phi),\\
O_p(m^{-1/2}), & g=g_c\text{ and the interior condition \eqref{eq:r0} holds},\\
O_p(m^{-1/4}), & g=g_c\text{ under the global threshold envelope}.
\end{cases}
\label{eq:root-m}
\end{equation}
\end{theorem}

The global $m^{-1/4}$ rate is a worst-case guarantee over correlation matrices;
the ordinary root-$m$ rate applies to every fixed threshold model satisfying
\eqref{eq:r0}.

\subsection{Selection Regret and Distinguishable Value Gaps}

Uniform value error translates directly into a decision guarantee.

\begin{corollary}[Selection regret and distinguishable value gaps]
\label{cor:regret}
Let
$\varepsilon_m=\sup_{S\in\Pi_{\mathcal B}}|\widehat\cI_g(S)-\cI_g(S)|$, let $S^*$
maximise $\cI_g$ over $\Pi_{\mathcal B}$, and let $\widetilde S$ be a feasible
protocol with empirical optimisation gap
\begin{equation}
\eta_m=\max_{S\in\Pi_{\mathcal B}}\widehat\cI_g(S)
-\widehat\cI_g(\widetilde S).
\label{eq:opt-gap}
\end{equation}
Then
\begin{equation}
\cI_g(S^*)-\cI_g(\widetilde S)\le2\varepsilon_m+\eta_m.
\label{eq:regret-opt}
\end{equation}
Exact empirical maximisation is the special case $\eta_m=0$.

For nested classes
$\Pi^{(1)}\subseteq\cdots\subseteq\Pi^{(L)}=\Pi_{\mathcal B}$, let
$\widehat S_\ell$ maximise $\widehat\cI_g$ over $\Pi^{(\ell)}$. On the event
$\sup_{S\in\Pi^{(\ell)}}|\widehat\cI_g(S)-\cI_g(S)|\le\varepsilon_\ell$,
\begin{equation}
\cI_g(S^*)-\cI_g(\widehat S_\ell)\le
\underbrace{\cI_g(S^*)-
\max_{S\in\Pi^{(\ell)}}\cI_g(S)}_{\text{class restriction}}
+\underbrace{2\varepsilon_\ell}_{\text{calibration error}}.
\label{eq:resolution-bound}
\end{equation}
\end{corollary}

An estimated gap larger than $2\varepsilon_m$ certifies the same ordering at
the population level, whereas a population gap larger than $2\varepsilon_m$
cannot be reversed by the plug-in estimates. Hence protocol-value differences
below this scale cannot be distinguished uniformly from calibration error.
Equation~\eqref{eq:resolution-bound} separates this uncertainty from the
approximation loss incurred by restricting the search to $\Pi^{(\ell)}$.

\theoremstyle{remark}
\newtheorem*{bestlinearremark}{Remark}
\begin{bestlinearremark}[Best-linear calibration]
The same perturbation and regret arguments apply to $\cI_{\mathrm L}$. If the
protocol covariances are uniformly nonsingular and their covariance and
target-cross-covariance blocks, together with the target variance, are
estimated uniformly at the root-$m$ rate, then
$\sup_{S\in\Pi_{\mathcal B}}|\widehat\cI_{\mathrm L}(S)-\cI_{\mathrm L}(S)|
=O_p(m^{-1/2})$. Appendix~\ref{sec:app-best-linear} states sufficient
conditions and proves this claim; \cref{cor:regret} then applies with
$\cI_g$ replaced by $\cI_{\mathrm L}$.
\end{bestlinearremark}

\section{Target-Aware Observation Design}
\label{sec:design}

Once candidate values have been identified and estimated accurately enough to
distinguish them, the remaining task is to choose a feasible protocol. Because
the target variance $V_g(K)$ does not depend on the protocol, maximising the
normalised value $\cI_g(S;K)=F_g(S;K)/V_g(K)$ is equivalent to maximising
$F_g(S;K)$.
Exact rank-one marginal gains support forward selection with one-swap
refinement, and exhaustive search measures optimisation error when the
catalogue is small enough. Appendix~\ref{sec:app-design} gives the derivations,
algorithmic details and comparator formulas.

\subsection{Feasible Protocols and Design Objective}

For design, an acquisition action is represented by
\begin{equation}
a=(\ell_a,\;r_a,\;c_a),
\label{eq:action}
\end{equation}
where $\ell_a\in\R^p$ specifies the linear measurement, $r_a\ge0$ is its
effective noise variance, and $c_a>0$ is its cost. The observation is
\begin{equation}
Y_{i,a}=\ell_a^\top Z_i+\varepsilon_{i,a},
\qquad \varepsilon_{i,a}\sim\cN(0,r_a),
\label{eq:observation}
\end{equation}
with noises independent across actions and independent of $Z_i$. Timing and
support are encoded in $\ell_a$: a point measurement uses a coordinate vector,
whereas a window average uses normalised weights over its support. Repetition
is encoded through $r_a$; for example, averaging $M_a$ independent measurements
with per-measurement variance $\nu_a^2$ gives $r_a=\nu_a^2/M_a$. The cost may
include all repetitions and other acquisition burdens.

For a finite catalogue $\cV$ and cost budget $\mathcal B$, the feasible family
is
\begin{equation}
\Pi_{\mathcal B}=\Big\{S\subseteq\cV:\
\textstyle\sum_{a\in S}c_a\le\mathcal B\Big\}.
\label{eq:feasible-family}
\end{equation}
Application-specific restrictions, such as choosing at most one precision
variant on a given support, can be imposed on the catalogue or by taking a
subfamily of $\Pi_{\mathcal B}$. The results below apply to any such subfamily.
We restrict attention to feasible protocols for which every non-empty
$\Sigma_S(K)$ is nonsingular; this is automatic when every selected action has
$r_a>0$.

At the population level, observation design is therefore the optimisation problem
\begin{equation}
\max_{S\in\Pi_{\mathcal B}}\;F_g(S;K),
\label{eq:design-problem}
\end{equation}
where the candidate set $\cV$ ranges over acquisition times, window lengths,
repetition counts and noise levels. Given calibration data, the empirical
version replaces $K$ by $\widehat K$ and maximises $F_g(S;\widehat K)$;
\cref{cor:regret} bounds what that substitution costs.

\subsection{Exact Marginal Gains}

Suppressing the argument $K$ in $Q_S(K)$, write
\begin{equation}
P_S:=K-Q_S(K)=\Cov(Z\mid Y_S).
\label{eq:residual-covariance}
\end{equation}
Thus $v_{a\mid S}=P_S\ell_a$ is the covariance between a new measurement and
the residual trajectory, and $s_{a\mid S}$ is its conditional variance,
including measurement noise.

\begin{proposition}[Rank-one marginal gains]
\label{prop:marginal}
Let $S=\varnothing$, or assume that $\Sigma_S(K)$ is nonsingular. Let
$a\notin S$ satisfy
$S\cup\{a\}\in\Pi_{\mathcal B}$ and $s_{a\mid S}>0$, and put
\begin{equation}
v_{a\mid S}=P_S\ell_a,
\qquad
s_{a\mid S}=\ell_a^\top P_S\ell_a+r_a.
\label{eq:vs}
\end{equation}
Then $Q_{S\cup\{a\}}=Q_S+v_{a\mid S}v_{a\mid S}^\top/s_{a\mid S}$ and
$P_{S\cup\{a\}}=P_S-v_{a\mid S}v_{a\mid S}^\top/s_{a\mid S}$, so that the exact
marginal gain is
\begin{equation}
\Delta_g(a\mid S)
=\sum_{j,k}\omega_j\omega_k
\left[C_g\!\left\{Q_{S,jk}+\frac{v_{a\mid S,j}v_{a\mid S,k}}{s_{a\mid S}}\right\}
-C_g(Q_{S,jk})\right].
\label{eq:marginal-nonlinear}
\end{equation}
For the mean target this collapses to the closed form
\begin{equation}
\Delta_{\mathrm{mean}}(a\mid S)
=\frac{(\omega^\top P_S\ell_a)^2}{\ell_a^\top P_S\ell_a+r_a}.
\label{eq:marginal-mean}
\end{equation}
\end{proposition}

\Cref{eq:marginal-nonlinear} is the exact objective increment rather than a
surrogate. The rank-one form lets search update the residual covariance without
re-forming the full observation-covariance inverse.

The marginal-gain objective is monotone under nested protocols, although it
need not exhibit diminishing returns.
\begin{lemma}[Monotonicity under nested protocols]
\label{prop:monotone}
For every $g\in L^2(\phi)$, every correlation matrix $K$ with nonsingular
measurement covariance for each non-empty protocol concerned, and every
$S\subseteq S'$,
$F_g(S;K)\le F_g(S';K)$.
\end{lemma}

Adding measurements cannot decrease $F_g$. Monotonicity, however, does not
imply diminishing returns: even the mean target reduces to $R^2$ subset
selection, which is not submodular in general \citep{das2018submodular}.

For noiseless point measurements and a linear integral target, the objective
reduces to kernel quadrature
\citep{bach2017equivalence,briol2019probabilistic}. Measurement noise and
nonlinear targets place the general problem outside that noiseless formulation.

\subsection{Search and Comparators}
\label{sec:algorithms}

Starting from $S=\varnothing$ and $P_S=K$, forward selection repeatedly adds
the feasible action with largest exact gain in
\eqref{eq:marginal-nonlinear}, using gain per unit cost when costs differ, and
then applies the rank-one update. A one-swap refinement makes feasible
selected--unselected exchanges whenever the exchange improves the objective.
Because monotonicity does not imply submodularity, exhaustive search is used as an
optimisation benchmark when the catalogue is small enough. Pseudocode and
operation counts are given in Appendix~\ref{sec:app-design-search}.

To isolate which information drives a schedule, we compare the target-aware
objective with several alternatives. Latent-state
mutual information maximises $I(Z;Y_S)$ and is target-free. Integrated
posterior variance minimises $\sum_j\omega_j(P_S)_{jj}$, so it uses temporal
weights but not the target transformation. Linear-target design retains the
actual weights and noise but sets $g(z)=z$, thereby maximising
$\omega^\top Q_S(K)\omega$. Noiseless kernel quadrature uses the same linear
criterion with $R_S=0$. These comparisons separate target nonlinearity from
the noise model and from target-free coverage criteria; their exact objectives
and marginal gains are recorded in Appendix~\ref{sec:app-design-comparators}.

\section{Empirical Evaluation}
\label{sec:synthetic}

\subsection{Experimental Overview}
\label{sec:empirical-overview}
The empirical evaluation addresses three questions that follow from the calibration and design results in Sections 4--5. First, with only finitely many dense calibration trajectories, how accurately can protocol values be estimated across a candidate family, and at what resolution can competing protocols be reliably distinguished and selected? Second, once those values are estimable, does designing measurements for the predictive value of the specified target improve on target-free or target-mismatched objectives, and how closely can practical search approach the exact target-specific optimum? Third, in fully annotated temporal data, are broad matched-budget differences in observation layout reproducible, and is there sufficient information to support the finer, data-dependent choice of exact observation locations?

The first two questions are studied under known data-generating laws. The
calibration experiments compare plug-in values with the exact population values
$\cI_g(S;K)$ and evaluate family-uniform error, selection regret and nested
candidate classes. The design experiments enumerate every feasible protocol in
moderate catalogues, so the target-specific optimum and the optimisation error
of forward selection with one-swap refinement are known exactly. The third
question is studied retrospectively in Sleep-EDF and Long-Term AF by
reconstructing partial observation protocols from complete annotations. Because
population protocol values are not observed in these data, performance is
measured by pooled held-out cross-fitted $R^2$; support selection and predictor
tuning are confined to the corresponding outer training folds. Simulation
specifications are given in Appendix~\ref{sec:appendix-simulations}, and the
real-data processing, cross-fitting and resampling procedures are given in
Appendix~\ref{sec:appendix-real}.

\subsection{Finite Calibration and Protocol Comparison}
\label{sec:simulation-evidence}
\label{sec:calibration-regret}

\Cref{fig:calibration} examines three consequences of finite calibration.
Panel (a) asks how accurately protocol values can be estimated over an entire
candidate family. Panel (b) asks whether the remaining estimation error is
large enough to affect selection of a near-optimal protocol. Panels (c,d)
then examine how calibration accuracy limits the useful granularity of the
candidate family. In all experiments, protocol values estimated from the
calibrated covariance $\widehat K$ are compared with their exact population
values under the known generating covariance. Full simulation specifications
are given in Appendix~D.

Panel (a) evaluates the family-uniform error
$\varepsilon_m=\max_{S\in\Pi_{\mathcal B}}
|\widehat{\cI}_g(S)-\cI_g(S;K)|$,
the largest protocol-value error within the candidate family.
This error decreases steadily as the number of calibration trajectories
increases. The full-range log--log slopes are \numSlopeMean{} and
\numSlopeOccZero; over the three largest calibration sizes they are
\numTailSlopeMean{} and \numTailSlopeOccZero, approaching the fixed-model
exponent $-1/2$ derived in \cref{sec:estimation}.

In this experiment, panel (b) shows that selecting a near-optimal protocol can
require substantially less accuracy than estimating every protocol value
uniformly. Across the six covariance--target cells, the mean cellwise log--log
slopes are \numRegretSlope{} for realised selection regret and
\numEpsSlope{} for the uniform-error envelope. The median and maximum
ratios of regret to this envelope are \numRegretRatioMedian{} and
\numRegretRatioMax. The difference reflects the two criteria being controlled:
$\varepsilon_m$ is determined by the largest estimation error anywhere in the
family, whereas near-optimal selection depends primarily on correctly ordering
the leading candidates.

\begin{figure}[H]
\centering
\includegraphics[width=0.82\linewidth]{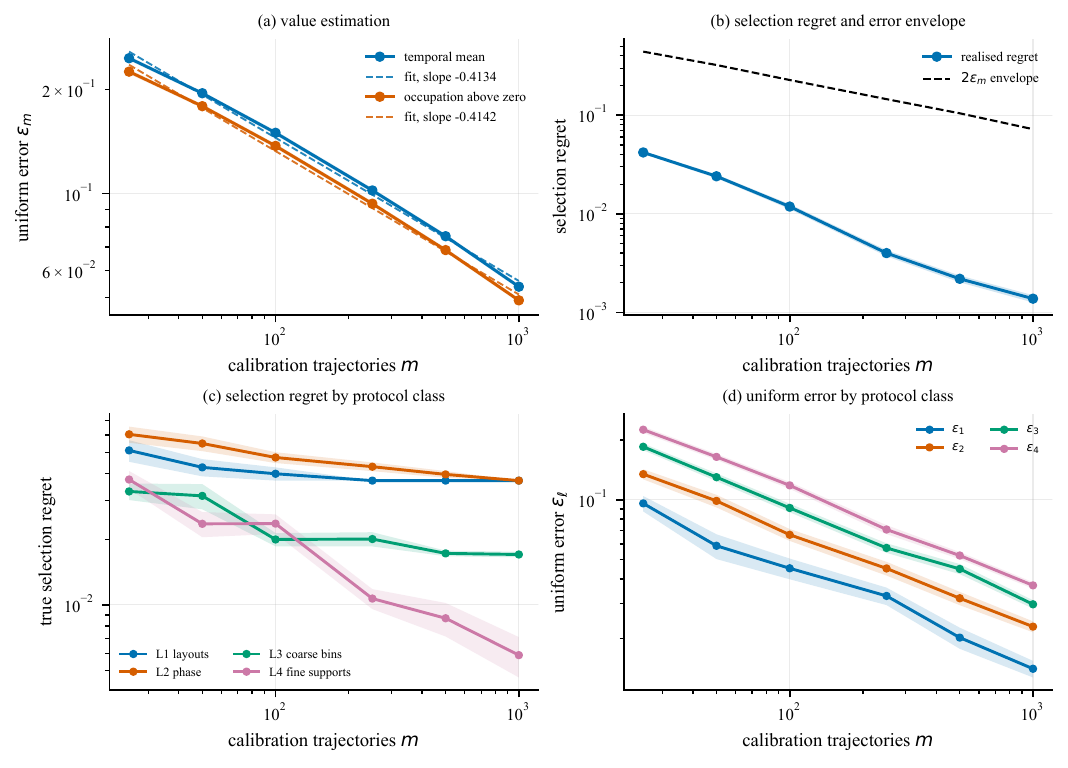}
\caption{Finite-calibration error, selection regret and candidate-family
resolution.
(a) Family-uniform protocol-value error $\varepsilon_m$ for the temporal-mean
and occupation-at-zero targets over all \numEstFamily{} four-action protocols
under a stationary OU Gaussian model on a $p=128$ grid.
(b) True regret of the protocol selected from estimated values, together with
the $2\varepsilon_m$ envelope from \cref{sec:estimation}, for four-action
selection from a 14-action catalogue, averaged over two covariance models and
three targets: the temporal mean and occupation above thresholds zero and one.
(c,d) Selection regret and family-uniform error for four nested candidate
families of increasing granularity, with sizes
\numResClassSizeOne, \numResClassSizeTwo, \numResClassSizeThree{} and
\numResClassSizeFour, under a nonstationary $p=64$ covariance.
Calibration sizes are
$m\in\{25,50,100,250,500,1000\}$.
Panels (a,b) use \numEstReplications{} replications and panels (c,d) use 30;
curves show replication means and bands show standard errors.}
\label{fig:calibration}
\end{figure}

Panels (c,d) show how this distinction affects the useful resolution of the
candidate family. At $m=25$, the finest class has selection regret
\numResRegretFineSmall, compared with \numResRegretCoarseSmall{} for the
coarsest class, despite a family-uniform error of \numResEpsFourSmall.
Thus, a richer candidate class can yield a better selected protocol even
before all of its values are estimated accurately. By $m=1000$, the uniform
error of the finest class falls to \numResEpsFourLarge, whereas the coarsest
class still incurs restriction regret \numResRegretCoarseLarge{} because
better protocols lie outside that class.

Together, the four panels distinguish uniform estimation resolution from
selection resolution. Reliable selection does not require every candidate
value to be estimated to the same precision; it requires sufficient
calibration to resolve the value differences among the candidates that compete
near the optimum. Candidate-family refinement can therefore be useful before
uniform estimation over the refined family becomes accurate.

\subsection{Target-Aware Design and Search}
\label{sec:target-design-experiment}

\Cref{tab:s5} examines two questions about observation design. First, how much
is lost when protocols are chosen using target-free or target-mismatched
objectives rather than the predictive value of the specified target? Second,
when the target-aware objective is used, how closely does the proposed search
procedure approach the exact optimum?

We evaluate \numDesignComparableInstances{} instances formed by five aggregate
targets and five covariance--action settings: stationary OU, stationary
Mat\'ern, horizon-varying covariance, recency-weighted prediction and
heterogeneous actions. The first four settings select four of 12 point actions;
the heterogeneous setting chooses among 18 actions under a cost budget of four.
Every feasible protocol is enumerated, providing the exact target-specific
optimum against which both alternative design objectives and approximate search
can be evaluated. Full simulation settings are given in Appendix~D.

For a method returning protocol $S$, we report relative efficiency
$\cI_g(S)/\cI_g(S^*)$, where $S^*$ is the exhaustive optimum for the target
being evaluated. An efficiency of one therefore indicates that the selected
protocol attains the target-specific optimum. Latent-state mutual information
and integrated posterior variance provide target-free comparators. The
linear-target comparator retains the actual temporal weights and measurement
noise but replaces the target transformation by $g(z)=z$, whereas kernel
quadrature uses the corresponding conventional noiseless integration
criterion.

\begin{table}[H]
\centering\small
\begin{tabular}{@{}lrrr@{}}
\toprule
Method & minimum & mean & median \\
\midrule
Latent-state mutual information & $\numDesignMethodMinMI$ & $\numDesignMethodMeanMI$ & $\numDesignMethodMedianMI$ \\
Integrated posterior variance & $\numDesignMethodMinIMSE$ & $\numDesignMethodMeanIMSE$ & $\numDesignMethodMedianIMSE$ \\
Actual-noise linear target & $\numDesignMethodMinLinear$ & $\numDesignMethodMeanLinear$ & $\numDesignMethodMedianLinear$ \\
Noiseless kernel quadrature & $\numDesignMethodMinKQ$ & $\numDesignMethodMeanKQ$ & $\numDesignMethodMedianKQ$ \\
Target-aware greedy & $\numDesignMethodMinGreedy$ & $\numDesignMethodMeanGreedy$ & $\numDesignMethodMedianGreedy$ \\
Target-aware greedy plus one-swap & $\numDesignMethodMinSwap$ & $\numDesignMethodMeanSwap$ & $\numDesignMethodMedianSwap$ \\
\bottomrule
\end{tabular}
\caption{Relative efficiency of alternative design objectives and search
procedures against the exhaustive target-specific optimum over
\numDesignComparableInstances{} kernel--action--target instances.}
\label{tab:s5}
\end{table}
\FloatBarrier

The choice of design objective matters most when temporal symmetry is broken.
In the stationary, uniformly weighted settings, supports optimised for
different targets transfer with little loss. Under nonstationarity,
non-uniform temporal weighting or heterogeneous actions, the target-specific
optimal supports become more distinct; the lowest cross-target efficiency in
the heterogeneous setting is \numDesignCrossHetero. The target-free
comparators show the same general distinction: good coverage of the latent
trajectory does not necessarily yield a protocol that is optimal for the
specified aggregate target.

The proposed search procedure is nevertheless close to the exhaustive
target-aware optimum. One-swap refinement raises the minimum efficiency in the
heterogeneous setting from \numDesignGreedyMinHetero{} to
\numDesignAwareMinHetero{} and reaches the exhaustive optimum there and in the
stationary OU and horizon-varying settings. In the remaining settings its
minimum efficiencies are \numDesignAwareMinMatern{} and
\numDesignAwareMinRecency, while evaluating only
\numDesignEvalFracHetero--\numDesignEvalFracStat{} of the exhaustive sets.

\subsection{Retrospective Protocol Comparison}
\label{sec:real}

The retrospective analyses use two fully annotated temporal data sets to
reconstruct partial observation protocols while computing the prediction target
from the complete analysable record. 

Sleep-EDF Expanded
\citep{kemp2000analysis,mourtazaev1995age,goldberger2000physiobank,
pollard2026physionet} contains 30-second sleep-stage annotations from the
Sleep Cassette (SC) and Sleep Telemetry (ST) studies. The analysed sample
contains \numSleepRecordings{} recordings from \numSleepSubjects{} subjects and
\numSleepHours{} hours of valid annotations. Stages 3 and 4 are combined as N3
\citep{rechtschaffen1968manual,iber2007aasm}. Each recording is mapped to
$p=\numSleepGrid$ relative-time anchors. The targets are the exact full-interval
REM, N3 and Wake proportions. The Long-Term AF Database
\citep{petrutiu2007abrupt,goldberger2000physiobank} contains
\numLtafRecords{} records and \numLtafHours{} hours of reviewed rhythm
annotations, with median analysable record length \numLtafMedianHours{} hours.
Analysis begins at the first rhythm marker, excluding any unannotated prefix.
Each record is mapped to $p=\numAfGrid$ equal relative-time bins, and the target
is the exact fraction of analysable annotated time spent in atrial fibrillation.

Sleep protocols expose stage-specific binary indicators at selected
relative-time anchors, with budgets $N\in\{4,8,16,32,64\}$ counting distinct
scored 30-second epochs. AF protocols expose bin-level AF fractions, with
budgets $N\in\{1,2,4,8,16,32\}$ observing the fraction $N/p$ of each analysable
record. Both analyses use five outer folds and report $R^2$ from pooled held-out
predictions. Sleep folds are subject-disjoint and stratified by source study,
with subject-balanced evaluation weights. Study, valid-duration and treatment
covariates enter the Sleep predictors unpenalised. AF uses record-level folds
and equal record weights because repeated-subject identifiers are unavailable.

At each budget, we compare a centred contiguous block with observations
dispersed as uniformly as possible over relative time. Sleep additionally
learns target-aware and kernel-quadrature supports within each outer training
fold. The target-aware criterion maximises a regularised best-linear estimate
of the fraction of residual target variance explained by the selected anchors.
The kernel-quadrature comparator uses the same repaired temporal covariance
without target outcomes. Support selection, ridge tuning and predictor fitting
are completed within the outer training fold. For fixed-template contrasts, conditional 2.5th--97.5th percentile ranges are
obtained by resampling held-out subjects for Sleep and records for AF while
keeping the fitted predictors fixed. Further details of the common-grid
mapping, weighting, support selection and resampling are given in
Appendix~\ref{sec:appendix-real}.

For REM, \cref{fig:real}(a) shows cross-fitted $R^2$ values of
\numCfContigFour, \numCfContigSixteen{} and \numCfContigSixtyFour{} for the
centred-contiguous template at $N=4,16,64$, compared with
\numCfUniformFour, \numCfUniformSixteen{} and \numCfUniformSixtyFour{} for
uniform dispersion. Across the 15 Sleep target--budget cells, the
dispersed-minus-contiguous percentile range has a positive lower endpoint in
\numSleepAdjustedPositiveRanges{} cells. The pooled Sleep results favour dispersion at several budgets, but the
source-specific contrasts are heterogeneous. In separate SC and ST analyses,
\numSleepCohortPositiveRanges{} of \numSleepCohortCells{}
source--target--budget ranges lie above zero,
\numSleepCohortNegativeRanges{} lie below zero and
\numSleepCohortUnresolvedRanges{} include zero. Additional analyses are
reported in Appendix~\ref{sec:appendix-real-sensitivity}.

For AF, at $N=4$ each template observes \numAfObservedPctFour{} of the
analysable record, equivalent to \numAfEquivalentHoursFour{} hour on a 24-hour
record. The contiguous and dispersed templates attain cross-fitted $R^2$
values \numCfAfContigOneHour{} and \numCfAfDispOneHour, respectively, with
difference \numAfBootDiffFour{} and conditional paired percentile range
\numAfBootRangeFour. At $N=16$, the corresponding values are
\numCfAfContigFourHour{} and \numCfAfDispFourHour, with difference
\numAfBootDiffSixteen{} and range \numAfBootRangeSixteen. The paired percentile range has a positive lower endpoint at every AF budget
containing at least two windows. At $N=1$, the difference is
\numAfBootDiffOne{} with range \numAfBootRangeOne. Because only one window is
observed, this comparison reflects window location rather than dispersion.
Outcome-defined cohort and grid-discretisation analyses are reported in
Appendix~\ref{sec:appendix-real-sensitivity}.

\FloatBarrier
\begin{figure}[htbp]
\centering
\includegraphics[width=0.86\linewidth]{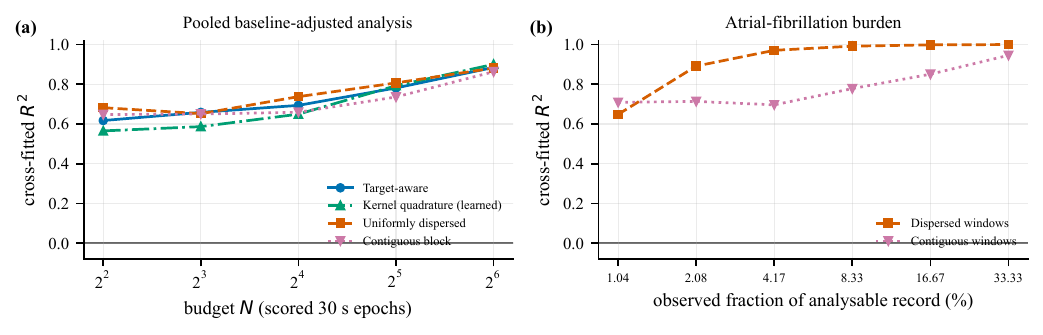}
\caption{Held-out prediction from reconstructed Sleep and AF observation
protocols. (a) Pooled baseline-adjusted cross-fitted $R^2$ for the Sleep
targets under target-aware, learned kernel-quadrature, uniformly dispersed and
contiguous supports. (b) Pooled cross-fitted $R^2$ for AF burden under
dispersed and contiguous window layouts.}
\label{fig:real}
\end{figure}
\FloatBarrier

We next examine the stability of the learned Sleep supports. At $N=16$, the
target-aware REM support attains baseline-adjusted cross-fitted
$R^2=\numCfAwareSixteen$, compared with \numCfKqSixteen{} for learned kernel
quadrature and \numCfUniformSixteen{} for fixed dispersion. The fold-averaged
pairwise Jaccard overlap among the REM-, N3- and Wake-specific supports is
\numSleepJaccardMin--\numSleepJaccardMax. The fitted supports therefore differ
substantially across targets, although this difference may also reflect
finite-sample variation in support selection.

\Cref{fig:sweep}(a) varies the number of subjects used for support selection
from \numSweepMSmall{} to the complete outer-training-fold sample while keeping
predictor fitting on the full training fold. The target-aware difference changes
from \numSweepDeltaKqSmall{} to \numSweepDeltaKqFull{} relative to learned
kernel quadrature and from \numSweepDeltaUniSmall{} to
\numSweepDeltaUniFull{} relative to fixed dispersion. Panel (b) repeats the complete selection and evaluation pipeline on
\numSelectionSubsampleReps{} study-stratified subsamples containing
\numSelectionSubsamplePct{} of subjects. The target-aware difference relative
to kernel quadrature has median \numSelectionDeltaKqMed{} and 2.5th--97.5th
percentile range
$[\numSelectionDeltaKqPlo,\numSelectionDeltaKqPhi]$. Relative to fixed
dispersion, the corresponding values are \numSelectionDeltaUniMed{} and
$[\numSelectionDeltaUniPlo,\numSelectionDeltaUniPhi]$. Both ranges include
zero.

\begin{figure}[htbp]
\centering
\includegraphics[width=0.84\linewidth]{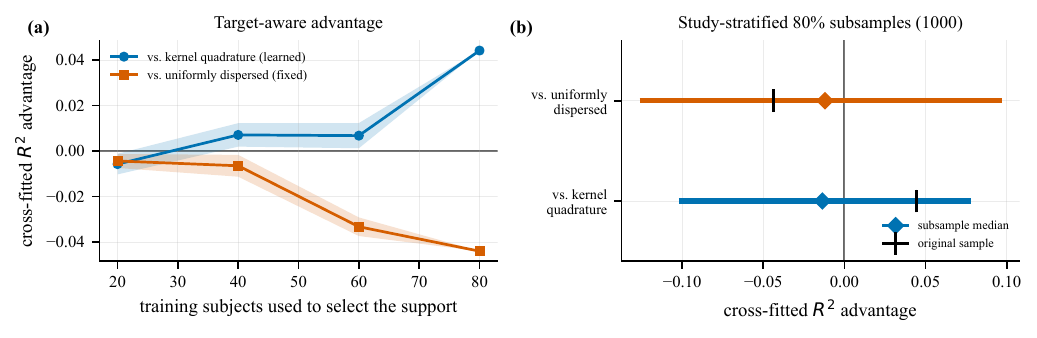}
\caption{Stability of the learned REM support at $N=16$.
(a) Target-aware held-out advantage as the support-selection sample increases.
(b) Target-aware advantage over learned kernel quadrature and fixed dispersion
across \numSelectionSubsampleReps{} study-stratified
\numSelectionSubsamplePct{} subject subsamples.}
\label{fig:sweep}
\end{figure}
\FloatBarrier

The real-data results distinguish broad temporal-layout contrasts more clearly
than exact learned support locations. Dispersion performs better than contiguous
observation in several pooled Sleep comparisons and across the multi-window AF
budgets, although the Sleep contrasts vary between source studies. The learned
Sleep supports are less stable across targets and subject subsamples, and their
paired held-out $R^2$ differences include zero. This agrees with the resolution
pattern in \cref{fig:calibration}: coarse protocol differences can be easier to
distinguish than fine support locations with finite data.

\section{Related Work}
\label{sec:related}

Counterfactual protocol evaluation intersects four literatures: learning from
incomplete temporal records, measurement design under an assumed process
model, evaluation across acquisition regimes, and statistical identification
of observation schemes. The main distinctions are whether the measurements
required by an alternative occur in the observed data and whether its value is
computable under a specified model or must be identified from the realised
measurement--target law.

\subsection{Learning from Incomplete Temporal Trajectories}
\label{sec:rw-incomplete}

Functional data analysis estimates population mean and covariance functions
from sparse or irregular measurements under sampling and smoothness
assumptions
\citep{ramsay2005functional,yao2005sparse,hall2006properties,li2010uniform}.
Longitudinal methods also allow observation times to depend on latent or
previously observed outcomes
\citep{pullenayegum2016irregular,weaver2023informative}. These estimators pool
within-subject pairs across units, so recovery of a covariance region requires
the sampling design to supply informative pairs there, together with the
assumptions imposed by the estimator. Under a single fixed protocol, times or
supports unique to an alternative may occur for no unit. The corresponding
covariance blocks are then not determined by the measurements alone; in the
linear-Gaussian analysis, the aggregate target adds the cross-covariance and
variance constraints characterised in Section~\ref{sec:identifiability}.

Machine-learning methods for irregular time series instead learn prediction or
interpolation rules directly from incomplete records. GRU-D models missingness
indicators and elapsed times, whereas latent ODEs and neural controlled
differential equations represent irregular observations in continuous time
\citep{che2018grud,rubanova2019latent,kidger2020neural,shukla2020survey}.
Early-prediction and active feature acquisition methods determine when an
observed prefix is sufficient or which feature to acquire next
\citep{xing2009early,shim2018joint}. Multiple-instance and label-proportion
methods also use aggregate labels, but retain a fixed observation scheme
\citep{dietterich1997multiple,quadrianto2009labels,zhou2018brief}. These methods
operate on observation patterns represented in the data. By contrast, an
alternative protocol here may require measurements absent for every unit in
the realised sample; its joint law with the target is therefore not available
as a marginal of the realised law.

\subsection{Measurement Design under an Assumed Model}
\label{sec:rw-design}

Classical optimal design allocates measurements to estimate parameters or
linear functionals efficiently under a statistical model
\citep{elfving1952optimum,pukelsheim2006optimal}. Bayesian and goal-oriented
design optimise expected information or posterior uncertainty about a
quantity of interest under a model or prior
\citep{lindley1956information,chaloner1995bayesian,attia2018goal,
alexanderian2021optimal,rainforth2024modern}. Active learning, sensor
selection, Gaussian-process sensor placement, and kernel quadrature likewise
optimise uncertainty, information gain, estimation error, or integration
accuracy under a fitted model, covariance, or kernel
\citep{mackay1992information,settles2012active,joshi2009sensor,
krause2008near,bach2017equivalence,briol2019probabilistic}.

Optimal design for longitudinal and functional data is especially close to the
present temporal setting: pilot data can estimate dependence and guide future
observation times, including designs for predicting a scalar response
\citep{ji2017optimal}. Such procedures ask which measurements to collect once
a model makes the criterion computable. Counterfactual protocol evaluation
asks the preceding question: whether the value in
Definition~\ref{def:values} is determined by the realised law at all. Once the
required dependence has been identified or calibrated, Section~\ref{sec:design}
becomes a target-aware sensor-selection problem. Its objective is the
population $R^2$ of the Bayes-optimal predictor of a fixed temporal aggregate,
not trajectory-reconstruction error or target-free information gain. For
noiseless point measurements and a linear aggregate it coincides with a
kernel-quadrature criterion; measurement noise and nonlinear aggregates move
beyond that case.

\subsection{Evaluation across Acquisition Regimes}
\label{sec:rw-evaluation}

Active feature acquisition performance evaluation (AFAPE) evaluates a new
acquisition policy from retrospective data generated by another acquisition
process \citep{vonkleist2025evaluation}. Depending on its assumptions, AFAPE
uses missing-data or offline-policy-evaluation methods, including direct,
inverse-probability-weighted, and doubly robust estimators. More generally,
off-policy evaluation requires overlap or related support conditions between
the behaviour and target regimes
\citep{precup2000eligibility,kallus2020double,sachdeva2020deficient}.

A fixed observation protocol can be represented as a deterministic acquisition
policy, but an alternative considered here may assign positive demand to
measurements having zero acquisition probability under the realised protocol.
Reweighting observed acquisition trajectories cannot identify their
distribution or predictive value. Identification must instead come from
restrictions linking observed and unobserved components of the same latent
trajectory to the aggregate target. Theorem~\ref{thm:minimal} and
Theorem~\ref{thm:impossibility} show that these links need not determine the
alternative's value, whereas Proposition~\ref{prop:augmentation} gives
conditions under which additional measurements remove the relevant ambiguity.
Graphical missing-data theory provides a related account of recoverability from
systematically incomplete observations
\citep{mohan2021graphical,nabi2020full}. Here, however, the acquisition rule is
fixed and known; the unresolved object is the latent temporal dependence
between measurements that were and were not collected.

\subsection{Identification, Comparison, and Predictive Limits of Observation Schemes}
\label{sec:rw-identification}

Statistical identification asks whether a functional is constant over all
latent models inducing the same observed law; partial identification retains
the resulting set of compatible values
\citep{koopmans1950identification,rothenberg1971identification,
manski2003partial,tamer2010partial}. We apply this principle to the predictive
value of an undeployed protocol. Definition~\ref{def:invisible} exposes the
observational-equivalence geometry in the linear-Gaussian model,
Theorem~\ref{thm:impossibility} gives a local non-identification certificate,
and Proposition~\ref{prop:permutation} gives an exact construction for
nonlinear aggregates. Proposition~\ref{prop:augmentation} conversely shows that
a specified protocol value can be identified without recovering the complete
latent covariance.

Bayes-error analyses characterise predictive limits for a fixed observed
feature--target law \citep{fukunaga1987bayes,ishida2023performance}. The present
question is whether that law determines the predictive limit under a different
feature-generating protocol. Blackwell's comparison of experiments and its
extensions order observation schemes across classes of decision problems
\citep{blackwell1953equivalent,lecam1990asymptotics,torgersen1991comparison};
protocol value is instead a scalar criterion for one fixed target and loss.
Transportability and data-fusion methods infer quantities outside a realised
regime through causal or invariance assumptions across populations or data
sources \citep{pearl2014external,bareinboim2016fusion}. Here the population and
target remain fixed, and the regime change concerns which measurements of the
same latent trajectory are collected.

The Gaussian identity used to evaluate a fixed protocol in
Proposition~\ref{prop:risk-identity} was derived by
\citet{zhang2026label}. The present work studies the identification of an
undeployed protocol's value, measurement augmentation that restores it, and
finite-calibration comparison and design.

\section{Discussion}
The observation protocol is part of a predictive system, not merely a fixed
preprocessing choice. When the target summarises a longer trajectory than the
measurements supplied to the learner, performance reflects both the learner
and the information admitted by the protocol. More training units and more
flexible predictors can reduce error under the realised protocol, but they do
not reveal the value of measurements that were never collected. Even the
population measurement--target law can leave dependence unresolved that is
decisive for an alternative protocol.

The practical meaning of ``more data'' therefore depends on how those data are
collected. If the ambiguity is invisible to the realised protocol, enlarging
the sample under that protocol reproduces the same constraints. A smaller
sample observed more extensively may be more useful. Targeted augmentation
removes ambiguity affecting one specified alternative, whereas dense
calibration supports a broader family. This motivates a large routine cohort
accompanied by a smaller, intensively observed calibration subset chosen with
future protocol decisions in mind.

Finite calibration determines how finely protocols can be compared. The
precision of an optimiser's output should not be confused with the precision
supported by the data. \Cref{fig:calibration} shows that richer candidate
classes can reduce selection regret before all candidate values are estimated
accurately, because selection depends mainly on ordering protocols near the
optimum. Useful granularity is set by the value gaps among competitive
protocols and the information available to resolve them.

The retrospective analyses illustrate this distinction. In Sleep, broad
dispersed-versus-contiguous contrasts are more reproducible than exact learned
supports, whose locations and held-out advantages vary across targets, source
studies and subsamples (\cref{fig:real,fig:sweep}). In AF, distributing a
matched budget across the record is substantially more predictive of
full-record burden than concentrating it in one block at the multi-window
budgets considered. This is not a general rule in favour of dispersion; equal
observation time need not carry equal information for a target defined over a
longer horizon. Protocol quality is also target-dependent: \cref{tab:s5}
shows that observing the process well and observing it well for a specified
aggregate can lead to different designs.

The same problem arises beyond temporal sampling. Short ECG recordings may not
reveal the cross-day dependence needed to evaluate dispersed monitoring; an
overlapping multimodal subset may be needed before imaging, proteomic or
wearable panels can be valued; and current environmental sensors may not reveal
the spatial dependence governing the value of new locations. These are not
ordinary feature-selection problems, because the proposed measurements may
never have been observed jointly with the target. Calibration precedes the
usual optimisation problem in longitudinal and sensor design
\citep{chaloner1995bayesian,ji2017optimal,krause2008near,joshi2009sensor}.

When point identification fails, the range of compatible values can still
support partial comparisons and robust design based on worst-case value or
minimax regret \citep{manski2003partial,tamer2010partial}. The Gaussian analysis
exposes this ambiguity through covariance directions, but the principle is
broader: an alternative is evaluable whenever its value is constant over the
latent laws compatible with the realised data. In more general models, the
corresponding task is to characterise that compatible set and collect
measurements that shrink the decision-relevant value range.

\section{Conclusion}
Counterfactual protocol evaluation asks a question prior to observation design: whether data collected under one protocol determine what could be achieved under another. The results show that this can fail even with unlimited data under the realised protocol, and that resolving the failure requires new kinds of observation rather than more samples of the same kind. Finite calibration then sets the scale at which candidate protocols can be compared reliably. Observation design should therefore be optimised only after its value is identified, and only at a resolution supported by the available data. Although developed for temporal protocols, the same principle applies to alternative measurement systems more generally.

\section*{Code and Data Availability}
Code and data are available at
\url{https://github.com/zxzok/cpv-explainer}. An interactive explainer is
available at \url{https://cpv.xizhe.net}.
\acks{The author declares no competing interests.}

\appendix
\section{Proofs for Value-Specific Identification}
\label{sec:app-identification}

\subsection{Gaussian Protocol-Value Identity}

\Cref{prop:risk-identity} is due to \citet{zhang2026label}. Draw posterior
replicates $Z^{(1)},Z^{(2)}$ independently given $Y_S$; their
cross-covariance is $Q_S(K)$. The Gaussian covariance identity and conditional
independence give
$\Cov\{\E[g(Z_j)\mid Y_S],\E[g(Z_k)\mid Y_S]\}=C_g\{Q_S(K)_{jk}\}$.
Weighted summation yields $F_g(S;K)$, while total variance gives $V_g(K)$ and
minimum Bayes error $V_g(K)-F_g(S;K)$, proving \eqref{eq:ceiling}.

\subsection{Minimal Stationary Counterexample}

Under the assumptions of \cref{thm:minimal}, stationarity and $\rho(0)=1$ give
the observable functionals
\begin{equation}
\begin{aligned}
\Var(Y_A)&=1+\nu^2,\\
\Cov(Y_A,\Theta)&=\frac1p\left(1+\sum_{l=1}^{p-1}\rho(l)\right),\\
\Var(\Theta)&=\frac1{p^2}\left(p+2\sum_{l=1}^{p-1}(p-l)\rho(l)\right).
\end{aligned}
\label{eq:stationary-observables}
\end{equation}
The last two functionals have Jacobian rows proportional to
$(1,\ldots,1)$ and $(p-1,\ldots,1)$. Hence
$\rank J=\min\{2,p-1\}$ and
$\dim\ker J=\max\{0,p-3\}$, proving parts (i)--(ii); at $p=4$ the kernel is
$\operatorname{span}\{(1,-2,1)\}$. Positive definiteness is an open condition,
so $K_\pm$ remain positive definite for all sufficiently small
$\varepsilon>0$.

The perturbation leaves
$b=\Cov(Z_1,\Theta)=\Cov(Z_2,\Theta)$ and $\Var(\Theta)$ unchanged. With
$R_B=\nu_B^2 I_2$, under $K_\pm$, $c=(b,b)^\top$ is an eigenvector of
$\Sigma_B^\pm$ with eigenvalue
$1+\nu_B^2+\rho(1)\pm\varepsilon$, giving
\eqref{eq:gap-closed-form}. Subtraction yields
\begin{equation}
\cI(B;K_-)-\cI(B;K_+)
=\frac{4b^2\varepsilon}
{\Var(\Theta)\left[\{1+\nu_B^2+\rho(1)\}^2-\varepsilon^2\right]}.
\label{eq:gap-value}
\end{equation}
Admissibility makes the denominator positive, so the gap is strict exactly
when $b\ne0$.

\subsection{General, Nonlinear and Augmented Protocols}

For \cref{thm:impossibility}, positive definiteness of $K_\varepsilon$ and
of $\Sigma_B(K_\varepsilon)$ persists for sufficiently small
$|\varepsilon|$, while $\diag(\Delta)=0$ preserves the unit diagonal.
Invisibility preserves the three Gaussian covariance blocks in
\eqref{eq:observables}, proving \eqref{eq:obs-equivalent}. Moreover,
$h^\top\Delta h=0$ makes the denominator of
$\cI(B;K_\varepsilon)$ constant. At zero, the inverse rule is
$\dd\Sigma_B(K_\varepsilon)^{-1}/\dd\varepsilon
=-W(B\Delta B^\top)W$. Applying the product rule to
$h^\top K_\varepsilon B^\top\Sigma_B(K_\varepsilon)^{-1}
BK_\varepsilon h$ gives \eqref{eq:ceiling-derivative}. Since the difference
between the values at $K_0\pm\varepsilon\Delta$ is
$2\varepsilon D_B(\Delta;K_0)+o(\varepsilon)$, a non-zero derivative gives
local separation.

For \cref{ex:three-point}, let $s=1+\nu_B^2$. Under $K_a$ and $K_a'$,
respectively, the posterior-replica covariance matrices for the one-point
alternative $B=e_1^\top$ are
\begin{equation*}
Q_B(K_a)=\frac1s(a,1,0)^\top(a,1,0),
\qquad
Q_B(K_a')=\frac1s(0,1,0)^\top(0,1,0).
\end{equation*}
Because $C_g(0)=0$, \eqref{eq:VF} gives
\begin{equation}
\cI_g(B;K_a)-\cI_g(B;K_a')
=\frac{C_g(a^2/s)+2C_g(a/s)}{3C_g(1)+2C_g(a)}.
\label{eq:three-point-gap}
\end{equation}
For nonconstant $g$, at least one nonconstant Hermite coefficient is non-zero,
so \eqref{eq:Cg} gives $C_g(r)>0$ for every $r>0$. The numerator and denominator
in \eqref{eq:three-point-gap} are consequently
positive for every $a\in(0,1)$ and every finite noise variance, proving the
universal strict inequality asserted in the example.

To prove \cref{prop:augmentation}, observe that the augmented law identifies
$\bar AK\bar A^\top$. If
$\rowspace(B)\subseteq\rowspace(\bar A)$, there is a matrix $C$ such that
$B=C\bar A$. It then determines
$BKB^\top=C(\bar AK\bar A^\top)C^\top$ and $BKh=C\bar AKh$,
and $h^\top Kh$ is already a block of the benchmark covariance. Together with
known $R_B$, these quantities determine the joint Gaussian law of
$(Y_B,\Theta)$ and hence its value, proving part (i). For part (ii), any
$\Delta\in\cD(A_{\rm aug},h)$ preserves the augmented benchmark law, and
\cref{thm:impossibility} applied with realised operator $A_{\rm aug}$ gives
the stated first-order failure certificate.

\section{Proofs for Finite Calibration}
\label{sec:app-calibration}

\subsection{Regularity of the Target Covariance Transform}

For \cref{prop:regularity}, if
$g-\E g=\sum_{k\ge1}(a_k/k!)H_k$, then
$C_g(r)=\sum_{k\ge1}(a_k^2/k!)r^k$. For
$g\in W^{1,2}(\phi)$, termwise differentiation and the Gaussian Sobolev
identity give
\begin{equation}
L_g:=\sup_{|r|\le1}|C_g'(r)|
\le \sum_{k\ge1}\frac{k a_k^2}{k!}
=\E\{g'(U)^2\}<\infty.
\label{eq:app-smooth-L}
\end{equation}
The non-negative coefficients show that equality in the supremum is attained
at $r=1$, proving the Lipschitz claim in part (i).

For $g_c(z)=\Ind\{z>c\}$, Plackett's identity gives, for $-1<r<1$,
\begin{equation*}
C_{g_c}'(r)=
\frac{\exp\{-c^2/(1+r)\}}{2\pi\sqrt{1-r^2}}.
\end{equation*}
Together with the sharp arcsine inequality
$|\arcsin b-\arcsin a|\le(\pi/\sqrt2)|b-a|^{1/2}$, this yields the global
bound
\begin{equation}
|C_{g_c}(r_2)-C_{g_c}(r_1)|
\le \frac{1}{2\sqrt2}|r_2-r_1|^{1/2}.
\label{eq:app-threshold-holder}
\end{equation}
At $c=0$ and $(r_1,r_2)=(-1,1)$ the constant is attained. Expanding the same
derivative as $r\uparrow1$ and integrating over $[1-\delta,1]$ gives
\eqref{eq:boundary-asymptotic}, so no larger global H\"older exponent is
possible. Finally, on $[-r_{\max},r_{\max}]$ the derivative is bounded by
\begin{equation}
L_{g_c}(r_{\max})=
\frac{\exp\{-c^2/(1+r_{\max})\}}
{2\pi\sqrt{1-r_{\max}^2}},
\label{eq:app-threshold-local-L}
\end{equation}
which proves part (iii).

\subsection{Covariance Repair and Resolvent Bounds}

For covariance repair, put $\bar K=\widehat\Sigma-\widehat R_0$ and
$e_0=\opnorm{\widehat\Sigma-(K+R_0)}+\opnorm{\widehat R_0-R_0}$.
If $0<\tau\le\lambda_{\min}(K)/2$ and
$e_0\le\min\{1/2,\lambda_{\min}(K)/2\}$, Weyl's inequality gives
$\bar K\succeq\tau I_p$, so flooring is inactive. For
$H=\Diag\{\diag(\bar K)\}^{-1/2}$, the unit diagonal of $K$ gives
$\opnorm{H-I_p}\le2e_0$ and $\opnorm H\le2$. Expanding
$H\bar K H-K=(H-I_p)\bar K H+(\bar K-K)H+K(H-I_p)$ yields
\begin{equation}
\opnorm{\cR_\tau(\bar K)-K}
\le(6\opnorm K+4)e_0
\le6(1+\opnorm K)e_0.
\label{eq:app-repair-bound}
\end{equation}
Thus diagonal rescaling preserves the raw rate at an interior $K$.

For the uniform resolvent bound, let $E=\widehat K-K$, $e=\opnorm E$, and
$\kappa=\max\{\opnorm K,\opnorm{\widehat K}\}$. Under
\eqref{eq:family-conditioning}, if $ae\le\lambda/2$, then for every non-empty
$S\in\Pi_{\mathcal B}$ the inverse and resolvent bounds are
\begin{equation*}
\begin{aligned}
\opnorm{(A_S\widehat K A_S^\top+R_S)^{-1}}&\le2/\lambda,\\
\opnorm{(A_S\widehat K A_S^\top+R_S)^{-1}
-(A_SKA_S^\top+R_S)^{-1}}&\le2ae/\lambda^2.
\end{aligned}
\end{equation*}
Expanding the two outer factors and the inverse in $Q_S$ then yields
\begin{equation}
\opnorm{Q_S(\widehat K)-Q_S(K)}\le L_Qe,
\qquad
L_Q=\frac{4\kappa a}{\lambda}
    +\frac{2\kappa^2a^2}{\lambda^2},
\label{eq:app-resolvent-bound}
\end{equation}
uniformly over the feasible family. Thus the constant worsens only through the
latent covariance scale, measurement amplification $a$, and protocol
conditioning $\lambda$.

\subsection{Value Error, Calibration Rates and Regret}

For \cref{thm:uniform-error}, the empty protocol has value zero under both
matrices. For non-empty $S$, non-negative weights summing to one, the
target modulus and \eqref{eq:app-resolvent-bound} imply
\begin{equation*}
|F_g(S;\widehat K)-F_g(S;K)|\le L(L_Qe)^\beta,
\qquad
|V_g(\widehat K)-V_g(K)|\le Le^\beta.
\end{equation*}
For the local threshold case, the second bound uses that $K$ and
$\widehat K$ have the same unit diagonal, so only off-diagonal arguments vary.
Since $0\le F_g(S;K)\le V_g(K)$, the ratio identity gives the more explicit
uniform bound
\begin{equation}
\sup_{S\in\Pi_{\mathcal B}}
|\widehat\cI_g(S)-\cI_g(S)|
\le
\frac{L(L_Q^\beta+1)e^\beta}{V_g(K)-Le^\beta}.
\label{eq:app-explicit-uniform}
\end{equation}
For sufficiently small $e$, the denominator is at least $v_0/2$ and
$\kappa$ is bounded by a fixed neighbourhood of $K$. Hence
\eqref{eq:app-explicit-uniform} is at most
$2L(L_Q^\beta+1)e^\beta/v_0$, which is \eqref{eq:uniform-error} with exactly
the dependencies stated there. For a threshold target satisfying
\eqref{eq:r0}, continuity of all finitely or uniformly conditioned resolvents
keeps the varying correlations in a common compact subinterval of $(-1,1)$;
the local Lipschitz constant \eqref{eq:app-threshold-local-L} then applies.

For \cref{thm:fixed-model}, fixed-dimensional Gaussian covariance
concentration gives
\begin{equation*}
\opnorm{\widehat\Sigma-(K+R_0)}=O_p(m^{-1/2}).
\end{equation*}
The assumed rate for $\widehat R_0$ therefore makes
$e_0=O_p(m^{-1/2})$. Because $K\succ0$ and
$\tau_m\downarrow0$, the conditions of \eqref{eq:app-repair-bound} hold with
probability tending to one, giving
$\opnorm{\widehat K-K}=O_p(m^{-1/2})$.

Substitution into \eqref{eq:app-explicit-uniform} gives root-$m$ value error
when $\beta=1$ and $m^{-1/4}$ under the global threshold envelope
$\beta=1/2$. At a fixed threshold model satisfying \eqref{eq:r0}, the local
Lipschitz modulus has $\beta=1$, restoring the root-$m$ rate. The family is
finite, so the conditioning and correlation neighbourhoods can be chosen
uniformly over all $S\in\Pi_{\mathcal B}$, which gives the stated rate.

For \cref{cor:regret}, insert the empirical value on both sides of the
comparison between $S^*$ and $\widetilde S$:
\begin{align*}
\cI_g(S^*)-\cI_g(\widetilde S)
&\le
\{\cI_g(S^*)-\widehat\cI_g(S^*)\}
+\{\widehat\cI_g(S^*)-\widehat\cI_g(\widetilde S)\}
+\{\widehat\cI_g(\widetilde S)-\cI_g(\widetilde S)\}\\
&\le2\varepsilon_m+\eta_m,
\end{align*}
which proves \eqref{eq:regret-opt}. Applying the same argument within
$\Pi^{(\ell)}$ gives
$\cI_g(\widehat S_\ell)\ge
\max_{S\in\Pi^{(\ell)}}\cI_g(S)-2\varepsilon_\ell$; subtracting this lower
bound from the global optimum gives \eqref{eq:resolution-bound}.

\subsection{Best-Linear Calibration}
\label{sec:app-best-linear}

\theoremstyle{plain}
\newtheorem*{bestlinearproposition}{Proposition}
\begin{bestlinearproposition}[Uniform root-$m$ error]
For non-empty $S$, write $\Sigma_S=\Var(Y_S)$ and
$c_S=\Cov(Y_S,\Theta)$, and let $v=\Var(\Theta)$. Suppose, uniformly over
$S\in\Pi_{\mathcal B}$, that $v\ge v_0>0$,
$\lambda_{\min}(\Sigma_S)\ge\lambda>0$, $\|c_S\|_2\le C_c<\infty$, and
\begin{equation*}
\delta_m:=\sup_{S\ne\varnothing}
\{\opnorm{\widehat\Sigma_S-\Sigma_S}+\|\widehat c_S-c_S\|_2\}
+|\widehat v-v|=O_p(m^{-1/2}).
\end{equation*}
Set both values of the empty protocol to zero; otherwise use
$\cI_{\mathrm L}(S)=c_S^\top\Sigma_S^{-1}c_S/v$ and its plug-in version.
Then
$\sup_{S\in\Pi_{\mathcal B}}|\widehat\cI_{\mathrm L}(S)-\cI_{\mathrm L}(S)|
=O_p(m^{-1/2})$.

In the measurement model, where $\Sigma_S=A_SKA_S^\top+R_S$ and
$c_S=A_Sc$ for $c=\Cov(Z,\Theta)$, it suffices that
$\sup_S\opnorm{A_S}<\infty$ and
\begin{equation*}
\opnorm{\widehat K-K}+\|\widehat c-c\|_2+|\widehat v-v|
+\sup_S\opnorm{\widehat R_S-R_S}=O_p(m^{-1/2}),
\end{equation*}
using $\widehat\Sigma_S=A_S\widehat K A_S^\top+\widehat R_S$ and
$\widehat c_S=A_S\widehat c$.
\end{bestlinearproposition}

\begin{proof}
On the event $\delta_m\le\min(\lambda/2,v_0/2)$, the resolvent identity
and $q_S=c_S^\top\Sigma_S^{-1}c_S$ give, uniformly over non-empty $S$,
\begin{equation*}
\begin{aligned}
\opnorm{\widehat\Sigma_S^{-1}}&\le\frac{2}{\lambda},
&\qquad
\opnorm{\widehat\Sigma_S^{-1}-\Sigma_S^{-1}}
&\le\frac{2\delta_m}{\lambda^2},\\
\sup_S|\widehat q_S-q_S|&\le
\frac{2\delta_m}{\lambda}(2C_c+\delta_m)
+\frac{2C_c^2\delta_m}{\lambda^2}
=O_p(m^{-1/2}).
\end{aligned}
\end{equation*}
Because $q_S\le C_c^2/\lambda$ and $\widehat v\ge v_0/2$, division by
$\widehat v$ and $v$ proves the claim. The primitive rate implies the assumed
one through
$\opnorm{\widehat\Sigma_S-\Sigma_S}
\le\opnorm{A_S}^2\opnorm{\widehat K-K}
+\opnorm{\widehat R_S-R_S}$ and
$\|\widehat c_S-c_S\|_2\le\opnorm{A_S}\|\widehat c-c\|_2$.
\end{proof}

\section{Design Derivations and Algorithms}
\label{sec:app-design}

\subsection{Rank-One Update}

For a non-empty $S$, write
$M=A_SKA_S^\top+R_S$, $b=A_SK\ell_a$ and
$d=\ell_a^\top K\ell_a+r_a$. The Schur complement of
$M_+=\left(\begin{smallmatrix}M&b\\b^\top&d\end{smallmatrix}\right)$ is
$d-b^\top M^{-1}b=s_{a\mid S}$, so the block-inverse formula
\citep[\S0.8]{horn2013matrix} gives
\begin{equation}
M_+^{-1}=
\begin{pmatrix}
M^{-1}+M^{-1}bs_{a\mid S}^{-1}b^\top M^{-1}
&-M^{-1}bs_{a\mid S}^{-1}\\
-s_{a\mid S}^{-1}b^\top M^{-1}&s_{a\mid S}^{-1}
\end{pmatrix}.
\label{eq:app-block-inverse}
\end{equation}
Substituting \eqref{eq:app-block-inverse} into the definition of $Q$ and
collecting the correction as an outer product yields
\begin{equation*}
Q_{S\cup\{a\}}-Q_S
=\frac{(K\ell_a-KA_S^\top M^{-1}b)
       (K\ell_a-KA_S^\top M^{-1}b)^\top}{s_{a\mid S}}
=\frac{P_S\ell_a\ell_a^\top P_S}{s_{a\mid S}}.
\end{equation*}
This is the stated $Q$ update; subtracting it from $K$ gives the $P$ update.
For $S=\varnothing$, the same formulas follow directly from
$Q_\varnothing=0$ and $P_\varnothing=K$. Substitution into \eqref{eq:VF}
gives \eqref{eq:marginal-nonlinear}, and setting $C_g(r)=r$ gives
\eqref{eq:marginal-mean}.

A forward pass uses rank-one $O(p^2)$ scores and costs
$O(d_{\mathcal B}|\cV|p^2)$, where
$d_{\mathcal B}=\max_{S\in\Pi_{\mathcal B}}|S|$. The implemented
swap sweep directly recomputes at most $|S||\cV|$ trial objectives, each in
$O(p^2|S|+|S|^3)$, and therefore costs
$O\{|S||\cV|(p^2|S|+|S|^3)\}$.

For \cref{prop:monotone}, $S\subseteq S'$ nests the two observation
$\sigma$-fields. Conditional expectation is an $L^2$ projection, so projection
contraction gives $F_g(S;K)\le F_g(S';K)$.

\subsection{Forward Selection and One-Swap Refinement}
\label{sec:app-design-search}

In exact-$d$ mode, assume that $\Pi_{\mathcal B}$ contains a protocol of
cardinality $d$. An addition $a\notin S$ is \emph{completion-preserving} if
there exists $S'\in\Pi_{\mathcal B}$ such that
$S\cup\{a\}\subseteq S'$ and $|S'|=d$. We call an addition \emph{admissible}
if it is completion-preserving in exact-$d$ mode, and if
$S\cup\{a\}\in\Pi_{\mathcal B}$ in at-most-$\mathcal B$ mode.

\begin{algorithm}[H]
\caption{Exact-gain forward selection with one-swap refinement}
\label{alg:forward-swap}
\begin{algorithmic}[1]
\Require $K$, $g$, $\omega$, catalogue $\cV$, family $\Pi_{\mathcal B}$; mode
exact-$d$ or at-most-$\mathcal B$
\State Initialise $S\gets\varnothing$, $Q_S\gets0$, $P_S\gets K$
\While{exact mode has $|S|<d$, or at-most mode has a positive-gain addition}
  \State choose the best admissible addition (gain per cost
  if costs differ), add it and update $Q_S,P_S$
\EndWhile
\While{a feasible improving swap exists and the swap cap is not reached}
  \State apply the first improving selected--unselected swap and update $Q_S,P_S$
\EndWhile
\State \Return $S$
\end{algorithmic}
\end{algorithm}

Exact mode continues to $|S|=d$; at-most mode stops when no feasible addition
has positive gain. Because $F_g$ need not be submodular, a zero-gain action can
enable a positive joint gain, so the latter rule is heuristic. Exhaustive optima are
reported for small families.

\subsection{Comparator Objectives}
\label{sec:app-design-comparators}

For the latent-state mutual-information comparator, when $R_S\succ0$ the
target-free objective and marginal gain are
\begin{align}
F_{\rm MI}(S)=I(Z;Y_S)
&=\frac12\log\det\!\left(I+R_S^{-1/2}A_SKA_S^\top R_S^{-1/2}\right),
\label{eq:mi-baseline}
\\[-2pt]
\Delta_{\rm MI}(a\mid S)
&=\frac12\log\frac{\ell_a^\top P_S\ell_a+r_a}{r_a},
\label{eq:mi-gain}
\end{align}
where the second line assumes $r_a>0$.

The integrated-posterior-variance comparator minimises
$\sum_j\omega_j(P_S)_{jj}$ and has marginal reduction
\begin{equation}
\Delta_{\rm IPV}(a\mid S)
=\sum_j\omega_j\frac{(P_S\ell_a)_j^2}{s_{a\mid S}}.
\label{eq:ipv-gain}
\end{equation}

The linear-target comparator replaces $g$ by the identity and maximises
$\omega^\top Q_S(K)\omega$. The kernel-quadrature comparator uses the same
linear criterion with $R_S=0$. Together, these objectives separate target
nonlinearity, measurement noise and target-free latent coverage.

\section{Simulation Settings}
\label{sec:appendix-simulations}

\subsection{Models, Targets and Protocol Families}

Unless stated otherwise, $Z$ is a standardised Gaussian trajectory on a
uniform grid. For lag $u$, the stationary correlations are OU,
$\rho(u)=e^{-u/\tau}$; Mat\'ern-3/2,
$\rho(u)=(1+2u/\tau)e^{-2u/\tau}$, scaled so that
$\int_0^\infty \rho(u)\dd u=\tau$; a weighted two-scale OU mixture; and the
damped-periodic kernel $e^{-u/6}\cos(2\pi u/3)$. A between-unit trait share
$\alpha$ gives
$K_{jk}=\alpha+(1-\alpha)\rho(|t_j-t_k|)$. The five targets in the design
comparison are the temporal mean, occupation above thresholds $0$ and $1.5$,
two-sided occupation outside $[-1.2,1.2]$, and the logistic aggregate
$g(z)=\{1+e^{-2z}\}^{-1}$. The calibration experiments use the target subsets
specified below.

\begingroup
\setlength{\intextsep}{5pt}
\begin{table}[H]
\centering\footnotesize
\setlength{\tabcolsep}{2pt}
\renewcommand{\arraystretch}{0.94}
\begin{tabular}{@{}>{\raggedright\arraybackslash}p{0.16\textwidth}
>{\raggedright\arraybackslash}p{0.36\textwidth}
>{\raggedright\arraybackslash}p{0.44\textwidth}@{}}
\toprule
Experiment & Latent and calibration model & Actions, family and repetitions \\
\midrule
Uniform error,
\cref{fig:calibration}(a)
& $T=20$, $p=128$, OU $\tau=1$, $\alpha=0$; known calibration-noise variance
$0.25$, action-noise variance $0.5$; mean and occupation-at-zero targets
& 12 equally spaced point actions, exactly four selected
($\numEstFamily$ protocols); $m\in\{25,50,100,250,500,1000\}$ and
$\numEstReplications$ replications. \\[2pt]

Selection regret,
\cref{fig:calibration}(b)
& $T=20$, $p=128$; trait--OU $(\alpha,\tau)=(0.2,2)$ and two-scale OU
$(\tau_1,\tau_2,w)=(0.5,6,0.5)$; noiseless calibration and action-noise
variance $0.5$; mean and occupation targets at thresholds zero and one
& Points at $\{1,3,\ldots,19\}$, width-three windows centred at 5 and 15,
and pre-specified three-replicate point packages at 2 and 18; exactly four of
14 actions; the same $m$ grid and replication count as panel (a). \\[2pt]

Nested protocol classes,
\cref{fig:calibration}(c,d)
& $T=20$, $p=64$; noiseless calibration; occupation-at-zero target;
action-noise variance $0.4$;
$K_{jk}=0.15+0.85\exp\{-|t_j-t_k|/[\tau(t_j)\tau(t_k)]^{1/2}\}$, with
$\tau(t)$ linear from $0.3$ to $3$
& Exactly four actions are selected. L1 compares contiguous with dispersed
layouts; L2 adds early, middle, late and full-horizon phases; L3 uses supports
on eight coarse bins; L4 adds 495 fine-grid four-point subsets. The cumulative,
deduplicated class sizes are \numResClassSizeOne, \numResClassSizeTwo,
\numResClassSizeThree{} and \numResClassSizeFour; the same $m$ grid and 30
replications. \\
\bottomrule
\end{tabular}
\caption{Calibration and nested protocol-class configurations.}
\label{tab:synthetic-calibration-config}
\end{table}
\endgroup

For the target-aware design comparison in \cref{tab:s5}, all settings use
$T=10$, $p=128$ and $\alpha=0$. The stationary OU and Mat\'ern-3/2 settings
use $\tau=1$ and action-noise variance $0.4$. The horizon-varying setting uses
\[
K_{jk}
=
\exp\{-|t_j-t_k|/[\tau(t_j)\tau(t_k)]^{1/2}\},
\qquad
\tau(t)=0.5+3.5t/10,
\]
with action-noise variance $0.15$. The recency-weighted OU setting uses
$\tau=1$, target weights proportional to $e^{-(T-t)/6}$ and action-noise
variance $0.4$. The heterogeneous-action setting uses OU scale $\tau=0.7$
and time-dependent action-noise variance $1+2t/10$.

The first four settings use 12 point actions at
$\{0.5,0.5+9/11,\ldots,9.5\}$ and select exactly four actions. The
heterogeneous family uses six centres $0.5+1.8k$, $k=0,\ldots,5$, widths
$\{0,1.5,3.5\}$, cost $1+0.15w$ and a cost budget of four. All feasible
protocols are enumerated to obtain the target-specific optimum. The linear
comparator uses the actual action noise, kernel quadrature its standard
noiseless criterion, and mutual information and integrated posterior variance
the same covariance and action catalogue. Greedy search starts from the empty
protocol; one-swap refinement takes the best feasible exchange until no
improvement remains or \numSwapRounds{} exchanges have been accepted.

\subsection{Numerical Evaluation of Nonlinear Targets}
\label{sec:appendix-numerical}

For one-sided threshold targets, $C_g$ is evaluated from the Plackett integral.
If $G_c(r)$ denotes the corresponding one-sided covariance transform, then the
two-sided occupation target uses
\begin{equation*}
C_{{\rm two},c}(r)=2\{G_c(r)+G_c(-r)\}.
\end{equation*}
The logistic transform is evaluated through a truncated Hermite expansion,
with coefficients computed by Gauss--Hermite quadrature. 

\section{Sleep-EDF and Long-Term AF Analysis Details}
\label{sec:appendix-real}
\suppressfloats[t]

\subsection{Annotation Mapping and Estimands}

Proportional binning maps unequal-length records to relative time. A Sleep
action reads the distinct midpoint epoch of its bin, so adjacent anchors need
not be consecutive raw epochs; an AF action averages one bin. With common-grid
trajectory $X_i\in[0,1]^p$, protocol $S$ observes $Y_{i,S}=A_SX_i$, while
$\Theta_i$ is the exact state proportion over the complete analysable interval,
not the anchor-grid mean.

Sleep retains pre- and post-sleep Wake and excludes movement and unknown epochs,
leaving \numSleepHours{} annotated hours (median \numSleepMedian{} per record);
night assignments follow \texttt{ST-subjects.xls}. The zero-cost baseline $C_i$
contains SC/ST, valid-stage duration and placebo/temazepam, with treatment zero
outside ST; neither it nor the validity mask is charged to the temporal budget.
All stages use $w_i^{(0)}=1/n_{s(i)}$ when subject $s(i)$ contributes $n_{s(i)}$
records. For AF, analysis starts at the first rhythm marker and excludes any
unannotated prefix from the target and candidate windows. Templates average $N$
bins, observing fraction $N/\numAfGrid$ despite record-specific elapsed time;
records are treated as the analysis units and receive equal weights.

The four quantities serve different roles. $\cI_g(S)$ is the population
protocol-value estimand, and $J_{\tau\mid C}(S)$ is the foldwise criterion used
to select Sleep supports. The pooled $\widehat R^2_{\rm cf}(S)$ reports
held-out predictive performance, while the full-sample
$\widehat{\cI}_{\rm L,p}(S)$ is used only for the AF grid-sensitivity analysis.

\subsection{Foldwise Selection and Held-Out Prediction}
\label{sec:appendix-real-prediction}

Within each Sleep outer-training fold, weighted least squares on an intercept
and $C$ gives residual anchors $X^\perp$, target $\Theta^\perp$ and moments
\begin{equation}
\widehat\Sigma_{0\mid C}=\Cov_w(X^\perp),\qquad
\widehat c_{\mid C}=\Cov_w(X^\perp,\Theta^\perp),\qquad
\widehat v_{\mid C}=\Var_w(\Theta^\perp).
\label{eq:real-moments}
\end{equation}
Writing $\widehat\Sigma_{0\mid C}=U\Diag(\widehat\lambda_j)U^\top$, the
fold-specific repair is
\begin{equation}
\widetilde\Sigma_\tau
=U\Diag\{\max(\widehat\lambda_j,\tau)\}U^\top,
\qquad \tau=q^{-1}\tr(\widehat\Sigma_{0\mid C})/p,
\label{eq:real-floor}
\end{equation}
where $q$ is the fold-specific number of independent training subjects (79--81;
median \numTabTrainSubjects); repeated nights do not increase it. Forward
selection followed by at most \numSwapRounds{} accepted one-swap updates
searches for a high-scoring support under
\begin{equation}
J_{\tau\mid C}(S)=
\frac{\widehat c_{\mid C}^\top A_S^\top
(A_S\widetilde\Sigma_\tau A_S^\top)^{-1}A_S\widehat c_{\mid C}}
{\widehat v_{\mid C}}.
\label{eq:real-objective}
\end{equation}
The target-agnostic learned comparator converts the same repaired residual
covariance to correlations and applies greedy search to the noiseless
kernel-quadrature criterion
for the uniform mean of standardised, nonconstant anchors, without
$\widehat c_{\mid C}$ or target outcomes. AF and fixed Sleep templates are
pre-specified.

Given a support, outer-training ridge with an intercept selects among 25
log-spaced penalties on $[10^{-6},10^2]$ by three-fold inner cross-validation,
grouped by subject for Sleep and record for AF. Centring and scaling use each
inner-training split and are recomputed on the full outer-training fold; Sleep
baselines are unpenalised, temporal features penalised, and predictions
unclipped.

Predictions from five outer test folds are pooled. With weighted target mean
$\bar\Theta_w$,
\begin{equation}
\widehat R^2_{\rm cf}(S)=1-
\frac{\sum_iw_i\{\Theta_i-\widehat\Theta_i^{(-f(i))}\}^2}
{\sum_iw_i(\Theta_i-\bar\Theta_w)^2}.
\label{eq:weighted-cf-r2}
\end{equation}
The statistic is computed once from all pooled held-out predictions rather than
by averaging foldwise values; as an out-of-sample $R^2$, it can be negative.

\subsection{Resampling and Sensitivity Analyses}
\label{sec:appendix-real-sensitivity}

Conditional 2.5--97.5 percentile ranges resample fixed-template held-out pairs
\numSleepFixedRangeResamples{} times by Sleep subject and
\numAfFixedRangeResamples{} times by AF record. SC/ST-stratified Sleep draws give
equal total weight per sampled subject. Because the fitted predictors remain
fixed, these ranges describe variation in the realised held-out pairs
conditional on those predictors. Pipeline stability is
assessed by \numSelectionSubsampleReps{} study-stratified samples of
\numSelectionSubsamplePct{} of subjects without replacement. Every subsample
reruns standardisation, moment estimation, support selection, inner tuning,
predictor fitting and held-out evaluation, while keeping each subject within
one outer and inner fold.

Alternative weighting and covariance-floor choices preserved the fixed-template ordering but changed the selected supports, consistent with the instability of fine support selection reported in the main analysis. Subject- versus recording-weighted moments preserve the fixed-template
score sign in all 15 cells but yield support overlap as low as
\numBalanceFineJaccardMin. Floor multipliers $\{0.5,1,2\}$ give maximum
outer-fold mean contrast range \numFloorCoarseSpread{} and overlap
\numFloorFineJaccardMin; the fixed-template ordering is stable while the
selected support varies.

SC-only and ST-only sensitivities rerun the full subject-grouped pipeline.
At $N=16$, \cref{tab:sleep-sensitivity} gives all target contrasts. For AF,
two outcome-defined cohorts assess sensitivity to excluding records with AF
burden at or near zero or one;
\cref{tab:afcohort} reports all fixed-template results at $N=4$ and $N=16$.
The main text summarises all 30 Sleep source cells by positive, negative and
unresolved ranges.

\begin{table}[H]
\centering\footnotesize
\setlength{\tabcolsep}{4pt}
\begin{tabular}{@{}lccc@{}}
\toprule
Target at $N=16$ & pooled adjusted & SC adjusted & ST adjusted \\
\midrule
REM & $\numSleepFullRemSixteenAdjusted$ & $\numSleepFullRemSixteenSC$ & $\numSleepFullRemSixteenST$ \\
N3 & $\numSleepFullNThreeSixteenAdjusted$ & $\numSleepFullNThreeSixteenSC$ & $\numSleepFullNThreeSixteenST$ \\
Wake & $\numSleepFullWakeSixteenAdjusted$ & $\numSleepFullWakeSixteenSC$ & $\numSleepFullWakeSixteenST$ \\
\bottomrule
\end{tabular}
\caption{Dispersed-minus-contiguous cross-fitted $R^2$ differences and
conditional 2.5--97.5 percentile ranges at the middle Sleep budget.}
\label{tab:sleep-sensitivity}
\end{table}

\begin{table}[H]
\centering\footnotesize
\setlength{\tabcolsep}{3.8pt}
\begin{tabular}{@{}lrrrrr@{}}
\toprule
& & \multicolumn{2}{c}{$N=4$ (\numAfObservedPctFour)}
& \multicolumn{2}{c}{$N=16$ (\numAfObservedPctSixteen)} \\
\cmidrule(lr){3-4}\cmidrule(lr){5-6}
Cohort & records & contiguous & dispersed & contiguous & dispersed \\
\midrule
All records & $\numAfCohortNAll$ &
$\numAfCohortAllContigFour$ & $\numAfCohortAllDispFour$ &
$\numAfCohortAllContigSixteen$ & $\numAfCohortAllDispSixteen$ \\
$0<\Theta_{\rm AF}<1$ & $\numAfCohortNStrict$ &
$\numAfCohortStrictContigFour$ & $\numAfCohortStrictDispFour$ &
$\numAfCohortStrictContigSixteen$ & $\numAfCohortStrictDispSixteen$ \\
$0.05\le\Theta_{\rm AF}\le0.95$ & $\numAfCohortNMixed$ &
$\numAfCohortMixedContigFour$ & $\numAfCohortMixedDispFour$ &
$\numAfCohortMixedContigSixteen$ & $\numAfCohortMixedDispSixteen$ \\
\bottomrule
\end{tabular}
\caption{Cross-fitted $R^2$ in the primary AF cohort and two
outcome-defined sensitivity cohorts.}
\label{tab:afcohort}
\end{table}

The AF grid diagnostic complements held-out prediction by measuring
discretisation sensitivity. For
$p\in\{64,128,256\}$, let $F_i^{(p)}$ be record $i$'s binwise raw AF fractions
and $\Theta_i^{(p)}=p^{-1}\mathbf 1^\top F_i^{(p)}$. On all
\numLtafCohortAll{} records, set $\Sigma_F=\Cov(F^{(p)})$,
$c_F=\Cov(F^{(p)},\Theta^{(p)})$, $v_F=\Var(\Theta^{(p)})$ and compute
\begin{equation*}
\widehat{\cI}_{\rm L,p}(S)=
\frac{c_F^\top A_S^\top(A_S\Sigma_FA_S^\top)^+A_Sc_F}{v_F}.
\end{equation*}
Without flooring, singular values at most \numLtafResolutionRcond{} times the
largest are discarded and windows use exact fractional overlap between their
continuous time intervals and the grid bins. The four diagnostics are a centred
one-hour contiguous window, a centred six-hour contiguous window, four uniformly
dispersed 15-minute windows and eight uniformly dispersed 15-minute windows.
Across these templates, the largest within-protocol range is
\numLtafResolutionSpread. This diagnostic quantifies numerical grid
sensitivity; \cref{fig:real}(b) and \cref{tab:afcohort} provide the held-out and
matched-budget comparisons.
\FloatBarrier

\begingroup
\small
\setlength{\baselineskip}{10.8pt}
\bibliography{references}
\endgroup

\end{document}